\documentclass[wcp]{jmlr}

\usepackage{longtable}

\usepackage{booktabs} 
\usepackage{enumitem}
\usepackage{float}
\usepackage{hyperref}
\usepackage{bbm}
\usepackage{mathtools}
\usepackage{algorithm}
\usepackage{algorithmic}
\usepackage{lipsum}
\usepackage{multirow}
\usepackage{diagbox}
\usepackage{comment}
\usepackage{thmtools,thm-restate}
\usepackage[font=small]{caption}
\DeclarePairedDelimiter{\norm}{\lVert}{\rVert}

\usepackage{rotating}
\usepackage{bm}
\usepackage{booktabs, makecell}
\usepackage{rotating}
\usepackage[export]{adjustbox}
\usepackage{microtype}
\newcommand*{\Scale}[2][4]{\scalebox{#1}{$#2$}}%

\usepackage{threeparttable}
\usepackage{tablefootnote}

\usepackage{enumitem,booktabs,cfr-lm}
\usepackage[referable]{threeparttablex}
\renewlist{tablenotes}{enumerate}{1}
\makeatletter
\setlist[tablenotes]{label=\tnote{\alph*},ref=\alph*,itemsep=\z@,topsep=\z@skip,partopsep=\z@skip,parsep=\z@,itemindent=\z@,labelindent=\tabcolsep,labelsep=.2em,leftmargin=*,align=left,before={\footnotesize}}
\makeatother

\makeatletter
\DeclareRobustCommand{\qed}{%
  \ifmmode 
  \else \leavevmode\unskip\penalty9999 \hbox{}\nobreak\hfill
  \fi
  \quad\hbox{\qedsymbol}}
\newcommand{\openbox}{\leavevmode
  \hbox to.77778em{%
  \hfil\vrule
  \vbox to.675em{\hrule width.6em\vfil\hrule}%
  \vrule\hfil}}
\newcommand{\qedsymbol}{\openbox}
\newenvironment{proof}[1][\proofname]{\par
  \normalfont
  \topsep6\p@\@plus6\p@ \trivlist
  \item[\hskip\labelsep\itshape
    #1.]\ignorespaces
}{%
  \qed\endtrivlist
}
\newcommand{\proofname}{Proof}
\makeatother
\usepackage{booktabs}

\makeatletter
\let\Ginclude@graphics\@org@Ginclude@graphics 
\makeatother

\jmlrvolume{157}
\jmlryear{2021}
\jmlrworkshop{ACML 2021}

\title[Calibrated Adversarial Training]{Calibrated Adversarial Training}

  \author{\Name{Tianjin Huang}$^1$ \Email{t.huang@tue.nl}\\
  \Name{Vlado Menkovski}$^1$ \Email{v.menkovski@tue.nl}\\
  \Name{Yulong Pei}$^1$    \Email{y.pei.1@tue.nl}\\
  \Name{Mykola Pechenizkiy}$^{1,2}$ \Email{m.pechenizkiy@tue.nl}\\
  \addr $^1$ Department of Mathematics and Computer Science, Eindhoven University of Technology, 5600 MB Eindhoven, the Netherland\\
  \addr $^2$ Faculty of Information Technology, University of Jyväskylä, 40100 Jyväskylä, Finland
 }

\editors{Vineeth N Balasubramanian and Ivor Tsang}

\begin{document}

\maketitle

\begin{abstract}
Adversarial training is an approach of increasing the robustness of models to adversarial attacks by including adversarial examples in the training set. One major challenge of producing adversarial examples is to contain sufficient perturbation in the example to flip the model's output while not making severe changes in the example's semantical content. Exuberant change in the semantical content could also change the true label of the example. Adding such examples to the training set results in adverse effects. In this paper, we present the Calibrated Adversarial Training, a method that reduces the adverse effects of semantic perturbations in adversarial training. The method produces pixel-level adaptations to the perturbations based on novel calibrated robust error. We provide theoretical analysis on the calibrated robust error and derive an upper bound for it. Our empirical results show a superior performance of the Calibrated Adversarial Training over a number of public datasets.
\end{abstract}

\section{Introduction}
\label{intro}
Despite the impressive success in multiple tasks, e.g. image classification~\cite{Krizhevsky2012,He2016},  object detection~\cite{Girshick2014}, semantic segmentation~\cite{Long2015}, deep neural networks (DNNs) are vulnerable to adversarial examples. In other words, carefully constructed small perturbations of the input can change the prediction of the model drastically~\cite{Szegedy2013,Goodfellow2014}. Furthermore, these adversarial examples have been shown high transferability, which greatly threat the security of DNN models~\cite{xie2019improving,huang2021direction}. This vulnerability of DNNs prohibits their adoption in applications with high risk such as autonomous driving, face recognition, medical image diagnosis.

In response to the vulnerability of DNNs, various defense methods have been proposed. These methods can be roughly separated into two categories: 1) certified defense, and 2) empirical defense. Certified defense tries to learn provable robustness against $\epsilon$-ball bounded perturbations~\cite{cohen2019certified,wong2018provable}. Empirical defense refers to heuristic methods, including augmenting training data~\cite{Madry2017} (e.g. adversarial training), regularization~\cite{Moosavi-Dezfooli2018,jakubovitz2018improving}, and inspirations from biology~\cite{tadros2019biologically}. Among all these defense methods, adversarial training has been the most commonly used defense against adversarial perturbations because of its simplicity and effectiveness~\cite{Madry2017,athalye2018obfuscated}. Standard adversarial training takes model training as a \emph{minmax} optimization problem (Section~\ref{sat})~\cite{Madry2017}. It trains a model based on on-the-fly generated adversarial examples $X'$  bounded by uniformly $\epsilon$-ball of input X (i.e.\ $\norm{X'-X} \leq \epsilon$). 

Although adversarial training is effective in achieving robustness, it suffers from two problems. Firstly, it achieves robustness with a severe sacrifice on natural accuracy, i.e.\ accuracy on natural images. Furthermore, the sacrifice will be enlarged rapidly when training with larger $\epsilon$. Secondly, there is an underlying assumption  that the on-the-fly generated adversarial examples within $\epsilon$-ball are semantic unchanged. However, recently,~\cite{guo2018low} and~\cite{sharma2019effectiveness} show that adversarial examples bounded by $\epsilon$-ball could be perceptible in some instances.~\cite{tramer2020fundamental} and ~\cite{jacobsen2018excessive} find that there are ``invariance adversarial examples'' for some instances, where ``invariance adversarial examples'' refer to those adversarial examples that model's prediction does not change while the true label  changes. All these findings indicate that this assumption does not consistently hold, which hurts the performance of the model.

In this paper, we first analyze the limitation for adversarial training and point out that some on-the-fly generated adversarial examples may be harmful for training models. For instance, in Figure~\ref{fig:1}, the adversarial examples for $x_{1}$ may be harmful since it crosses the oracle classifier's decision boundary. To address the limitation, we propose a calibrated adversarial training, which is derived on the upper bound of a new definition of robust error (Calibrated robust error). Calibrated adversarial training is composed of weighted cross-entropy loss for natural input and $\mathbf{KL}$ divergence for calibrated adversarial examples where calibrated adversarial examples are pixel-level adapted adversarial examples in order to reduce the adverse effect of adversarial examples with underlying semantic changes.  

Specifically, our contributions are summarized as follows:
\begin{itemize}
    \item Theoretically, we analyze the limitation for adversarial training, and propose a new definition of robust error: Calibrated robust error. Furthermore, we derive an upper bound for the calibrated robust error.
    \item We propose the calibrated adversarial training based on the upper bound of calibrated robust error, which can reduce the adverse effect of adversarial examples.  
    \item Extensive experiments demonstrate that our method achieves the best performance on both natural and robust accuracy among baselines and provides a good trade-off between natural accuracy and robust accuracy. Furthermore,
    it enables training with larger perturbations, which yields higher adversarial robustness.
\end{itemize}

\begin{figure}[htb]
\centering
\vspace{-0.2in}
\includegraphics[width=0.6\textwidth]{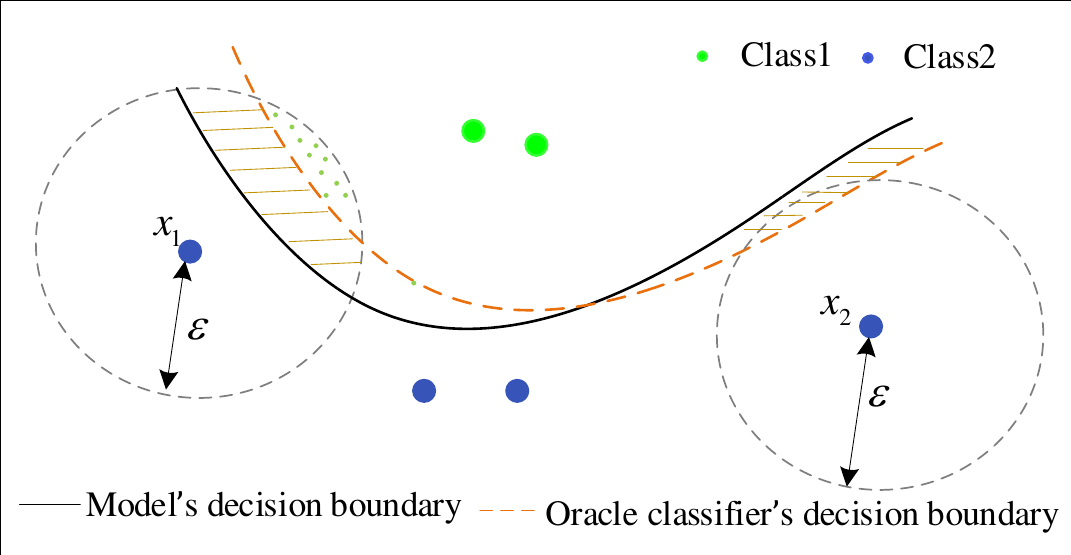}
\caption{Illustration for neighborhoods of inputs and the decision boundaries.}
\label{fig:1}
\end{figure}

\section{Related Work}
Many papers have proposed their variants of adversarial training for achieving either more effective adversarial robustness or a better trade-off between adversarial robustness and natural accuracy. Generally, they can be categorized into two groups. The first group is to adapt a loss function for outer minimization or inner maximization. For instance,~\cite{kannan2018adversarial} introduces a regularization term to enclose the distance between adversarial example and corresponding natural example.~\cite{pmlr-v97-zhang19p} proposes a theoretically principled trade-off method (Trades).~\cite{ding2019mma} proposes Max-Margin adversarial (MMA) training by maximizing the margin of a classifier.~\cite{wang2019improving} proposes MART by introducing an explicit regularization for misclassified examples.~\cite{wu2020adversarial} proposes Adversarial Weight Perturbation (AWP) for regularizing the weight loss landscape of adversarial training. 
~\cite{andriushchenko2020understanding,huang2020bridging} propose FGSM adversarial training + gradient-based regularization for achieving more effective adversarial robustness. The other group is to generate adversarial examples with adapted perturbation strength. Our work belongs to this group. Several recent works including Customized adversarial training~\cite{cheng2020cat}, Currium adversarial training~\cite{cai2018curriculum}, Dynamic adversarial training~\cite{wang2019convergence}, Instance adapted adversarial training~\cite{balaji2019instance}, Adversarial training with early stopping (ATES)~\cite{sitawarin2020improving}, Friendly adversarial training (FAT)~\cite{pmlr-v119-zhang20z}, heuristically propose to adapt $\epsilon$ in instance-level for adversarial examples.

\section{Preliminary}
\label{preliminary}
\subsection{Notations}
We denote capital letters such as $X$ and $Y$ to represent random variables and lower-case letters such as $x$ and $y$ to represent realization of random variables. We denote by $x \in \mathcal{X}$ the sample instance, and by $y \in \mathcal{Y}$ the label, where $\mathcal{X}\in \mathbb{R}^{m\times n}$ indicates the instance space. We use $\mathcal{B}(x,\epsilon)$ to represent the neighborhood of instance $x$: $\{x':\norm{x' -x}_p \leq \epsilon \}$. We denote a neural network classifier as $f_{\theta}(x)$, the cross-entropy loss as $L(\cdot)$ and Kullback-Leibler divergence as $\mathbf{KL}(\cdot||\cdot)$. We denote $P(Y|X)$ as probability output after softmax and $P(Y=y|X)$ as the probability of $Y=y$. $sgn(\cdot)$ denotes the sign function and  $f_{oracle}$ denotes the oracle classifier that maps any inputs to correct labels.

\subsection{Standard Adversarial Training} \label{sat}
Given a set of instance $x \in \mathcal{X}$ and $y \in \mathcal{Y}$. We assume the data are sampled from an unknown distribution $(X,Y) \sim \mathcal{D}$. The standard adversarial training can be formally expressed as follows~\cite{Madry2017}:
\begin{align}
\small
 \min_{\theta}\rho(\theta),  \rho(\theta)=\mathbbm{E}_{(X,Y)\sim D}[\max_{X' \in \mathcal{B}(X,\epsilon)} L(f_{\theta}(X'),Y)].
 \label{eq:at}
\end{align}

\subsection{Projected Gradient Descent (PGD)}
\cite{Madry2017} utilizes projected gradient to generate perturbations. Formally, with the initialization $x^{0}=x$, the perturbed data in $t$-th step $x^{t}$ can be expressed as follows:
\begin{equation}
x^{t}=\Pi_{\mathcal{B}(x,\epsilon)}(x^{t-1}+\alpha\cdot sgn(\nabla_x L(f_{\theta}(x^{t-1}),y))) ,
\label{pgd}
\end{equation}
where $\Pi_{\mathcal{B}(x,\epsilon)}$ denotes projecting perturbations into the set $\mathcal{B}(x,\epsilon)$, $\alpha$ is the step size and $t \in \{1,2,...,T\}$. We denote PGD attack with $T=20$ as PGD-20 and $T=100$ as PGD-100. 

\subsection{C\&W attack}
Given $x$, C\&W attack~\cite{carlini2017towards} searches adversarial examples $\Tilde{x}$ by optimizing the following objective function:
\begin{align}
     \norm{\Tilde{x}-x}_{p}+c \cdot h(\Tilde{x}),
\end{align}
with
\begin{align}
h(\Tilde{x})= \max(\max_{i\ne t}f_{\theta}(\Tilde{x})_{i}-f_{\theta}(\Tilde{x})_{t},-k), \notag
\end{align}
 where $c>0$ balances the two loss terms and $k$ encourages adversarial examples to be classified as target $t$ with larger confidence. This paper adopts C\&W$_{\infty}$ attack and follows the implementation in \cite{pmlr-v97-zhang19p,cai2018curriculum} where they replace cross-entropy loss with $h(\Tilde{x})$ in PGD attack.

\subsection{Robust Error} \label{robusterr}
We introduce the definition of robust error given by ~\cite{pmlr-v97-zhang19p,schmidt2018adversarially}.
\begin{restatable}[Robust Error~\cite{pmlr-v97-zhang19p,schmidt2018adversarially}]{definition}{defre}
Given a set of instance $x_1,...,x_n \in \mathcal{X}$ and labels $y_1,...,y_n \in \{-1,+1\}$. We assume that the data are sampled from an unknown distribution $(X,Y) \sim D$. The robust error of a classifer $f_{\theta}:\mathcal{X} \xrightarrow{} \mathbf{R}$ is defined as: $\mathcal{R}_{rob}(f):=\mathbf{E}_{(X,Y) \sim D} \bm{1} \{\exists X' \in \mathcal{B}(X,\epsilon)\;s.t. \; f_{\theta}(X')Y \leq 0 \} $.
\end{restatable}

\section{Method}\label{methods}

\subsection{Analysis For Adversarial Training}\label{AAT}

Current adversarial training including its variants trains a model by minimizing robust error directly, which may hurt the performance of the model. Taking standard adversarial training as an example, it firstly approximates robust error by the inner maximization and then minimizes the approximated robust error. However, the on-the-fly adversarial examples generated by the inner maximization could be semantically damaged for some instances, e.g., in Figure~\ref{fig:1}, the semantical content of the adversarial examples for $x_{1}$ could be damaged since it crosses the decision boundary of $f_{\theta}$. Therefore, the objective function (Eq.~\ref{eq:at}) can be decomposed into two terms according to the oracle classifier's decision boundary:
\begin{align}
\small
 &\min_{\theta}\rho(\theta),\  \rho(\theta) = \mathbbm{E}_{(X,Y)\sim D}[ \overbrace{\max_{\delta \in \mathcal{B}(X,\epsilon)} L(f_{\theta}(X+\delta),Y) \mathbf{1}\{f_{oracle}(X+\delta)=Y\}}^{(a)} \notag \\
 &+\overbrace{\max_{\delta \in \mathcal{B}(X,\epsilon)} L(f_{\theta}(X+\delta),Y) \mathbf{1}\{f_{oracle}(X+\delta) \ne Y\}}^{(b)}].
 \label{eq:at_deco}
\end{align}
The term (b) contributes to negative effects since the cross-entropy loss takes $Y$ as the label of adversarial examples $X+\delta$ while the true label of $X+\delta$ is not $Y$. This term is equivalent to bringing noisy labels in training data, which also explains why a large perturbation magnitude in adversarial training setting will lead to a severe drop in natural accuracy of model.

To address this drawback, we propose calibrated robust error and build our defense method based on it.

\subsection{Calibrated Robust Error}\label{CaliRE}

\begin{restatable}[Calibrated Robust Error (Ours)]{definition}{defcre}
Given a set of instances $x_1,...,x_n \in \mathcal{X}$ and labels $y_1,...,y_n \in \{-1,+1\}$. We assume that the data are sampled from an unknown distribution $(X,Y) \sim D$. Assume there is an oracle classifier $f_{oracle}$ that maps any input $x \in \mathbf{R}^d $ into its true label.  The calibrated robust error of a classifier $f_{\theta}:\mathcal{X} \xrightarrow{} \mathbf{R}$ is defined as: $\mathcal{R}_{cali}(f):=\mathbf{E}_{(X,Y) \sim D} \bm{1} \{\exists X' \in \mathcal{B}(X,\epsilon)\; s.t. \; f_{\theta}(X')f_{oracle}(X') \leq 0 \} $.
\end{restatable}

\begin{restatable}[]{thm}{theorencaliRE}\label{theo_calire}
Given a set of instance $x_1,...,x_n \in \mathcal{X}$, a classifier $f_{\theta}:\mathcal{X} \xrightarrow{} \mathbf{R}$ and an oracle classifier $f_{oracle}$ that maps any input $x \in \mathbf{R}^d$ into its true label and assumed the decision boundaries of $f_{\theta}$ and $f_{oracle}$ are not overlapped~\footnote{Not overlapped denotes $f_{\theta}$ and $f_{oracle}$ are not exactly the same.}, we have:
\begin{align}
  \mathcal{R}_{rob}(f) \leq \mathcal{R}_{cali}(f).
\end{align}
\end{restatable}
The proof can be found in Appendix~\ref{proofR}. From Theorem~\ref{theo_calire}, it can be observed that minimizing robust error can be obtained by minimizing calibrated robust error.

\subsection{Upper Bound on Calibrated Robust Error}
 In this section, we derive an upper bound on calibrated robust error.
\begin{restatable}[Upper Bound]{thm}{upperbound}
Let $\psi$ be a nondecreasing, continuous and convex function:$[0,1] \xrightarrow{} [0,\infty]$. Let $\mathcal{R}_{\phi}(f):=\mathbf{E}{\phi}(f_{\theta}(X)Y)$ and $\mathcal{R}_{\phi}^{*}:=\min_{f}\mathcal{R}_{\phi}(f)$, $\mathcal{R}(f):=\mathbf{E}(f_{\theta}(X)Y)$ and $\mathcal{R}^{*}=\min_{f}\mathcal{R}(f)$. For any non-negative loss function $\phi$ such that $\phi(0)\geq 1$, any measurable $f_{\theta}: \mathcal{X} \xrightarrow{}\mathbf{R}$ and any probability distribution on $\mathcal{X} \times \{+1,-1\}$, we have:
\begin{align}
   & \mathcal{R}_{cali}(f) -\mathcal{R}^{*} \leq   \psi^{-1} (\mathcal{R}_{\phi}(f)-\mathcal{R}_{\phi}^{*})+\mathbf{E} \biggl[\max_{\substack{X'\in \mathbf{B}(X,\epsilon)\\ f_{oracle}(X')=Y}} \phi(f_{\theta}(X')Y)\biggr].
\end{align}
\end{restatable}
The proof can be found in Appendix~\ref{proofTheo}. From the upper bound, it can be observed:
\begin{itemize}
    \item If the oracle classifier's decision boundary crosses $\epsilon$-ball, the upper bound is decided by the adversarial examples that close to the oracle classifier's decision boundary. If the oracle classifier's decision boundary does not cross $\epsilon$-ball, the upper bound is decided by the adversarial examples that close to the boundary of $\epsilon$-ball. 
    \item Minimizing $ \mathcal{R}_{\phi}(f) + \mathbf{E} \biggl[\max_{\substack{X'\in \mathbf{B}(X,\epsilon)\\ f_{oracle}(X')=Y}} \phi(f(X')Y)\biggr]$ can reduce the calibrated robust error. From Theorem~\ref{theo_calire}, we can know that calibrated robust error is the upper bound of robust error. Therefore, it also reduces the robust error of the model.   
\end{itemize}
\subsection{Method for Defense} \label{method}
From the upper bound, we define the general objective function as follows:
\begin{align}
    \min_{\theta} \mathbf{E} \biggl[\phi(f_{\theta}(X)Y)+\max_{\substack{X'\in \mathbf{B}(X,\epsilon)\\ f_{oracle}(X')=Y}} \phi(f_{\theta}(X')Y)\biggr]. \label{primalobj}
\end{align}
The analysis for the difference between Eq.~\ref{primalobj} and the general objective function~\cite{pmlr-v97-zhang19p} derived on the upper bound of robust error can be found in Appendix~\ref{anlgeneralobj}.

The first term in Eq.~\ref{primalobj} is the surrogate loss of misclassification on natural data, and we design it as cross-entropy weighted by $(1-predicted\; probability)$. Formally, it is expressed as:
\begin{align}
    \phi(f_{\theta}(X)Y)= L(f_{\theta}(X),Y)\cdot(1-P(Y=y|X)).   \label{natual loss}
\end{align}

The second term in Eq.~\ref{primalobj} is the surrogate loss on adversarial examples. However, it can not be solved directly since $f_{oracle}$ is unknown. Therefore, we propose a approximate solution with two steps. Firstly, we generate adversarial examples based on $\max_{X'\in \mathbf{B}(X,\epsilon)} \phi(f_{\theta}(X')Y)$. Secondly, we adapt the adversarial examples in pixel-level such that it approximately satisfies the constraint $f_{oracle}(X')=Y$ and we name the pixel-level adapted adversarial examples as \textbf{calibrated adversarial examples}. We rewrite $\max_{\substack{X'\in \mathbf{B}(X,\epsilon)\\ f_{oracle}(X')=Y}} \phi(f_{\theta}(X'),Y)$ as follows:
\begin{align}
&    X'=X+\delta=argmax_{X'\in \mathbf{B}(X,\epsilon)} \phi(f_{\theta}(X')Y) \label{eq.6} \\
&    X_{cali}'=X+M\odot \delta,\; M \in  \mathbb{R}^{m \times n},   M[i,j] \in (0,1),  \label{eq.7}
\end{align}
where the $\odot$ denotes Hadamard product. From Eq.~\ref{eq.6} and Eq.~\ref{eq.7}, we can see that calibrated adversarial examples $X_{cali}'$ are obtained by adapting adversarial perturbations with soft mask $M$. Please refer to Section~\ref{visualizationM} and Appendix~\ref{difference} for better understanding how does the mask $M$ adapt the adversarial perturbations. $\delta$ can be solved by various adversarial attacks, e.g., PGD attack. Therefore, the problem of the inner maximization in Eq.~\ref{primalobj} is transformed to find a proper soft mask $M$. Considering that soft mask $M$ relies on input $X$ and perturbation $\delta$, we propose to learn it by a neural network $g_{\varphi}$, which is defined as follows:
\begin{align}
    M=g_{\varphi}(X,\delta).
\end{align}
Therefore, by replacing $\phi(f_{\theta}(X)Y)$ with Eq.~\ref{natual loss} and $X'$ with $X_{cali}'$, the objective function (Eq.~\ref{primalobj}) is transformed to follows:

\begin{align}
\min_{\theta} \mathbf{E}_{(X,Y) \sim D} [L(f_{\theta}(X),Y)\cdot(1-P(Y=y|X)) + \beta \cdot \phi(f_{\theta}(X_{cali}')Y)] \label{f_obj0},
\end{align}
where $X_{cali}'$ is solved by Eq.~\ref{eq.7}, and $\beta$ is a hyper-parameter for balancing two terms. In practice, we follow~\cite{pmlr-v97-zhang19p,wang2019improving} to use $\mathbf{KL}$ divergence for the surrogate loss $\phi(\cdot)$ in the outer minimization step. Thus, Eq.~\ref{f_obj0} can be reformulated as follows: 
\begin{align}
    \min_{\theta} \mathbf{E}_{(X,Y) \sim D} [L(f_{\theta}(X),Y)\cdot(1-P(Y=y|X))  + \beta \cdot \mathbf{KL}(P(Y|X_{cali}')||P(Y|X))].  \label{f_objective}
\end{align}
From Eq.~\ref{f_objective}, it can be observed that there are two main differences with other variants of adversarial training, e.g., AT, Trades, MART, etc. (See Appendix~\ref{loss_variants} for the detailed descriptions of their loss functions.):
\begin{itemize}
    \item We use weighted cross-entropy loss instead of cross-entropy loss in order to make the loss function pay more attention to misclassified samples.
    \item The $\mathbf{KL}$ divergence is based on calibrated adversarial examples that reduce the adverse of some adversarial examples because calibrated adversarial examples are expected to be satisfied with $f_{oracle}(X_{cali}')=Y$.
\end{itemize}

Finally we design the objective function for $g_{\varphi}(X,\delta)$ based on the two constraints: (1) $X_{cali}'$ should be close to $X'$ as far as possible in order to keep the inner maximization constraint in Eq.~\ref{primalobj}. (2) $X_{cali}'$ is expected to be satisfied with $f_{oracle}(X_{cali}')=Y$.
Therefore, the objective function for $g_{\varphi}(X,\delta) $ is designed as follows:
\begin{align}
\small
 \min_{\varphi} \mathbf{E}_{(X,Y) \sim D}[  \mathbf{KL}(P(Y|X_{cali}') || P(Y|X'))  +\beta_{1} \cdot L(f_{\theta}(X_{cali}'),Y)],     \label{g_objective}
\end{align}
 where $\mathbf{KL}$ divergence term corresponds to the constraint (1) and cross-entropy loss $L(\cdot)$ corresponds to the constrain (2). $\beta_{1}$ is the hyper-parameter that controls the strength of the constraint (2).
 
 We denote our method as calibrated adversarial training with PGD attack (CAT$_{cent}$) if $X'$ is solved by PGD attack, calibrated adversarial training with C\&W$_{\infty}$ attack (CAT$_{cw}$) if $X'$ is solved by  C\&W$_{\infty}$ attack.

\section{Experiments}
In this section, we first conduct extensive experiments to assess the effectiveness of our approach in achieving natural accuracy and adversarial robustness, then we conduct experiments for understanding the proposed method.
\subsection{Evaluation on Robustness and Natural Accuracy}
\subsubsection{Experimental settings}
Two datasets are used in our experiments: MNIST~\cite{lecun1998mnist}, and CIFAR-10~\cite{krizhevsky2010cifar}. For MNIST, all defense models are built on four convolution layers and two linear layers. For CIFAR-10, we use PreAct ResNet-18~\cite{He2016} and WideResNet-34-10~\cite{zagoruyko2016wide} models. The architectures of auxiliary neural network $g_{\varphi}$ for MNIST and CIFAR-10 can be found in Appendix~\ref{imple}.  Following previous researches~\cite{pmlr-v97-zhang19p,wu2020adversarial}, Robustness is measured by robust accuracy against white-box and black-box attacks. For white-box attack, we adopt PGD-20/100 attack~\cite{Madry2017}, FGSM attack~\cite{Goodfellow2014} and C\&W$_{\infty}$~\cite{carlini2017towards}. For black-box attack, we adopt a query-based attack: Square attack~\cite{andriushchenko2020square}. 

\textbf{Baselines}
Standard adversarial training and the three latest defense methods are considered: 1)TRADES~\cite{pmlr-v97-zhang19p}, 2)MART~\cite{wang2019improving}, 3)FAT~\cite{pmlr-v119-zhang20z}. The detailed descriptions of baseline methods can be found in Appendix~\ref{imple}.

\textbf{Hyper-parameter settings}\label{Hyperps}
During training phase, for MNIST, we set $T =20$, $\epsilon=0.3$, $\alpha = \epsilon/T$ for the training attack, and set $\beta = 1$, $ \beta_{1}=0.3$ by default. For CIFAR-10, we set $T=10$, $\alpha=2/255$, $\epsilon=8/255$ for the training attack and set $\beta = 5$ by default. We train models with $\beta_{1}=0.05,0.1,0.3$ respectively. For all baselines, they are trained using the official code that their authors provided and the hyper-parameters for them are set as per their original papers. More training details are introduced in Appendix~\ref{imple}.  

During test phase, for MNIST, we set $\epsilon = 0.3$ and $\alpha = 0.015$ for PGD attack. For CIFAR-10, we set $\epsilon = 8/255$ and $\alpha = 0.003$ for PGD attack. And we follow the implementation in \cite{pmlr-v119-zhang20z} for C\&W$_{\infty}$ attack where $\epsilon=0.031$, $\alpha=0.003$, $T=30$ and $k=50$.

Note that during the training process, we use the PGD attack with random start, i.e.\ adding random perturbation of $[-\epsilon,\epsilon]$ to the input before PGD perturbation. But for the test in our experiments, we use PGD attack without random start by default~\footnote{We find that PGD attack (restart=1) without random start is stronger than that with random start.}.

\subsubsection{Evaluation on White-box Robustness}
This section shows the evaluation on white-box attacks. All attacks have full access to model parameters. We first conduct an evaluation on a simple benchmark dataset: MNIST and then conduct an evaluation on a complex dataset: CIFAR-10.

\textbf{MNIST} 
Table~\ref{tab:mnist} reports natural accuracy and robust accuracy under PGD-20 and PGD-100 respectively. For baselines, we do not include results from FAT and MART since they do not provide training code for MNIST. From Table~\ref{tab:mnist}, we can see that the proposed method can achieve higher natural accuracy and robust accuracy compared with standard adversarial training. 
Besides, we notice that with larger $\epsilon=0.4$, adversarial robustness can be boosted further by our defense method.

\textbf{CIFAR-10}
We evaluate the performance based on two benchmark architectures, i.e., PreAct ResNet-18 and WideResNet-34-10. All defense models are tested under the same attack settings as described in Section~\ref{Hyperps} except for \emph{FAT} on WideResNet-34-10 since this evaluation is copied from their paper directly where it is evaluated with $\epsilon=0.031$ for PGD attack. Table~\ref{tab:RN_CIFAR} and Table~\ref{tab:WRN_CIFAR} report natural accuracy and robust accuracy on the test set. ``Avg'' denotes the average of natural accuracy and all robust accuracy, and it indicates the overall performance on both natural accuracy and robust accuracy. For our method, we report mean + standard deviation with 5 repeated runs. 

From Table~\ref{tab:RN_CIFAR} and Table~\ref{tab:WRN_CIFAR}, it can be seen that our method achieves the best performance on both natural accuracy and robust accuracy under all attacks except for FGSM among baselines. Moreover, with $\beta_{1}=0.3$, our method improves natural accuracy with a large margin while keeps comparable performance with baselines on robust accuracy. Besides, our method achieves high ``Avg'' value, which indicates our method has a good trade-off between natural accuracy and robust accuracy. Finally, we observe that the robustness achieved by our method has smaller accuracy under stronger attacks, i.e.\ PGD-100 and CW$_{\infty}$, than weaker attacks, i.e.\ FGSM and PGD-20. It indicates that the robustness achieved by our method is not caused by ``gradient masking''~\cite{athalye2018obfuscated}. 

Experiments on CIFAR-100 can be found in Appendix~\ref{exp100}.

\begin{table}[t]
\caption{Evaluation on MNIST. The value besides model name denotes the max perturbation magnitude used in the training phase. -: denotes the training loss fails in decrease. We report mean with 5 repeated runs and skip the standard deviations since they are small ($<0.4$\%), which hardly affects the results.}
\vspace{-0.2 in}
\begin{center}
\begin{small}
\begin{sc}
\begin{threeparttable}
\begin{tabular}{l|ccc}
\toprule
Models & Natural & PGD-20 & PGD-100 \\
\midrule
AT($0.3$) & 99.2& 93.4&92.3  \\
AT($0.4$) & -& -& -\\
TRADES($0.3$)\tnote{*}&99.3&94.9&92.9\\
TRADES($0.4$)\tnote{*}&99.1&95.3&91.6\\
CAT$_{cent}$($0.3$)&\textbf{99.3}&95.4 &93.2 \\
CAT$_{cent}$($0.4$) & 99.2& 96.8& 95.8 \\
CAT$_{cw}$($0.3$)&99.1 &96.2 &95.0  \\
CAT$_{cw}$($0.4$) & 99.1& \textbf{97.1}& \textbf{96.2}\\
\bottomrule
\end{tabular}
\begin{tablenotes}\footnotesize
\small
\item[*] Model is trained with $\beta=1.0$.
\end{tablenotes}
\end{threeparttable}
\end{sc}
\end{small}
\end{center}
\label{tab:mnist}
\vskip -0.1in
\end{table}

\begin{table*}[t]
\caption{Evaluation on CIFAR-10 for PreAct ResNet-18 under white-box setting.}
\vspace{-0.2 in}
\begin{center}
\begin{small}
\begin{sc}
\begin{tabular}{l|cccccc}
\toprule
Models & Natural &FGSM  & PGD-20 & PGD-100&CW$_{\infty}$&Avg \\
\midrule
AT & 83.0& 57.3& 52.9&51.9 &50.9 &59.2\\
TRADES($\beta:6$) &82.8 & 57.6& 52.8& 51.7&50.9&59.2\\
MART ($\lambda:5$) & 83.0& \textbf{60.2}&53.9 &52.3&49.9 &59.9\\
FAT($\beta$:6) & 85.1& 58.3& 52.1&50.5 &50.4 &59.3\\
CAT$_{cent}$($\beta_{1}:0.05$)& 84.1 $\pm$ 0.3&59.5 $\pm$ 0.2 & \textbf{55.6 $\pm$ 0.3}&\textbf{54.9$\pm$0.3}&50.8$\pm$0.2&\textbf{61.0} \\
CAT$_{cent}$($\beta_{1}:0.1$) & 85.9 $\pm$0.2&58.5$\pm$0.3 &54.1 $\pm$0.1& 53.4 $\pm$0.06&50.44$\pm$0.3 &60.4\\
CAT$_{cw}$($\beta_{1}:0.05$) & 84.2 $\pm$0.3&58.9$\pm$0.2 &55.3 $\pm$0.4& 54.5 $\pm$0.5&\textbf{51.3$\pm$0.3} &60.9\\
CAT$_{cw}$($\beta_{1}:0.1$) & 85.1 $\pm$0.5&58.9$\pm$0.3 &54.9 $\pm$0.5& 54.1 $\pm$0.4&51.2$\pm$0.1 &60.8\\
\toprule
CAT$_{cent}$($\beta_{1}:0.3$) & 88.0 $\pm0.2$&57.0$\pm$0.4 &51.1 $\pm0.5$& 49.9 $\pm0.4$&47.8$\pm$0.2 &58.8\\
CAT$_{cw}$($\beta_{1}:0.3$) & \textbf{88.1 $\pm0.1$}&57.4$\pm0.5$ &51.5 $\pm0.1$& 50.1 $\pm0.2$&48.8$\pm$0.2 &59.2\\
\bottomrule
\end{tabular}
\end{sc}
\end{small}
\end{center}
\label{tab:RN_CIFAR}
\vskip -0.1in
\end{table*}

\begin{table*}[t]
\caption{Evaluation on CIFAR-10 for WideResNet-34-10 under white-box setting.}
\vspace{-0.2 in}
\begin{center}
\begin{small}
\begin{sc}
\begin{tabular}{l|cccccc}
\toprule
Models & Natural&FGSM  & PGD-20 & PGD-100&CW$_{\infty}$&Avg \\
\midrule
AT & 86.1& 61.8& 56.1&55.8 &54.2 &62.8\\
TRADES($\beta:6$)  & 84.9& 60.9& 56.2& 55.1&54.5& 62.3\\
MART ($\lambda:5$) & 83.6& 61.6&57.2&56.1&53.7&62.5\\
FAT($\beta$:6) & 86.6$\pm$0.6& 61.9$\pm$0.6& 55.9$\pm$0.2&55.4$\pm$0.3 &54.3$\pm$0.2&62.8\\
CAT$_{cent}$($\beta_{1}:0.05$)& 86.6$\pm$0.1 & 60.9 $\pm$ 0.1& 57.7 $\pm$ 0.1& 57.2 $\pm$0.2&53.9 $\pm$0.6& 63.3\\
CAT$_{cent}$($\beta_{1}:0.1$)&87.5$\pm$0.51 & 61.5 $\pm$0.5&57.2 $\pm$0.3& 56.6 $\pm$0.4&54.0$\pm$0.4 &63.4\\
CAT$_{cw}$($\beta_{1}:0.05$)&86.4$\pm$0.1 & \textbf{62.7 $\pm$0.2}&\textbf{59.7 $\pm$0.1}& \textbf{58.7 $\pm$0.3}&\textbf{56.0$\pm$0.1} &\textbf{64.7}\\
CAT$_{cw}$($\beta_{1}:0.1$)&87.4$\pm$0.1 & 62.3 $\pm$0.1&58.6 $\pm$0.2& 57.3 $\pm$0.19&55.6$\pm$0.07 &64.2\\
\toprule
CAT$_{cent}$($\beta_{1}:0.3$)&$88.9\pm0.4$ &59.8 $\pm$0.6&54.8 $\pm$0.7&53.9 $\pm$0.6&51.6$\pm0.2$ &61.8\\
CAT$_{cw}$($\beta_{1}:0.3$)&\textbf{89.3$\pm$0.1} & 60.8$\pm$0.27&55.1$\pm$0.3&53.2$\pm$0.5&52.6$\pm$0.4 &62.2\\
\bottomrule
\end{tabular}
\end{sc}
\end{small}
\end{center}
\label{tab:WRN_CIFAR}
\vskip -0.1in
\end{table*}
\subsubsection{Evaluation on Black-box Robustness}
We conduct evaluation on black-box settings. We choose to use Square attack~\cite{andriushchenko2020square} in our experiments. Square attack is a query-efficient black-box attack, which has been shown that it achieves white-box comparable performance and resists ``gradient masking''~\cite{andriushchenko2020square}. In our experiments, we set hyper-parameters $n_{queries}=5000$ and $eps=8/255$ for Square attack. The experiments are carried out on CIFAR-10 test set based on PreAct ResNet-18 and WideResNet-34-10 architectures. Results are showed in Table~\ref{tab:black}. It can be seen that our method achieves the best accuracy among all baselines under square attack. Besides, by comparing Table~\ref{tab:black} with Table~\ref{tab:RN_CIFAR} and Table~\ref{tab:WRN_CIFAR}, we can find that accuracy under black-box attack is lower than under white-box attack like PGD and CW$_{\infty}$ attacks. It demonstrates that adversarial robustness achieved by our method is not due to ``gradient masking ''~\cite{athalye2018obfuscated}. 

\subsection{Understanding the Proposed Defense Method}
\subsubsection{Visualization of Soft Mask $M$}\label{visualizationM}
We visualize the learned soft mask $M$ for further understanding calibrated adversarial examples. As showed in Figure~\ref{Fig:visMask}, natural images are randomly selected from MNIST, and adversarial examples are generated by PGD-20 attack with $\epsilon=0.4$. Soft masks and calibrated adversarial examples are generated accordingly. It can be observed that soft masks have high values on the background but have low values on the digit, which indicates that they try to reduce perturbations on the digit. Furthermore, by comparing calibrated adversarial examples with adversarial examples, we find that pixel values on digits for calibrated adversarial examples tend to be homogeneous, which is more consist with them on natural images. In other words, soft masks try to prevent adversarial examples from breaking semantic information that could impact the performance of the model.

\begin{figure}[htb]
\begin{minipage}[b]{0.49\textwidth}
\captionof{table}{Evaluation on CIFAR-10 for PreAct ResNet-18  and WideResNet-34-10 under black-box setting. -: Not Available.}
\label{tab:black}
\begin{small}
\begin{sc}
\begin{tabular}{l|cc}
\toprule
Models &ResNet &WRN \\
\midrule
AT & 55.12& 59.19\\
TRADES& 54.85&59.0   \\
MART & 54.98&  57.7 \\
FAT & 55.35& - \\
CAT$_{cent}$($\beta_{1}:0.05$) &56.4$\pm$0.1&59.1$\pm$0.5\\
CAT$_{cent}$($\beta_{1}:0.1$)   &56.4$\pm$0.1 &59.6$\pm$0.8\\
CAT$_{cw}$($\beta_{1}:0.05$)  &56.3 $\pm$0.2&60.9$\pm$0.1 \\
CAT$_{cw}$($\beta_{1}:0.1$)  &\textbf{56.5 $\pm$0.1}& \textbf{60.9$\pm$0.2}\\
\bottomrule
\end{tabular}
\end{sc}
\end{small}
\end{minipage}
\begin{minipage}[b]{0.49\textwidth}
\begin{flushright}
\setlength\tabcolsep{1pt}
\settowidth\rotheadsize{Radcliffe Cam}
\begin{tabular}{l ccccc}
\rothead{\centering{Natural}} 
                        &   \includegraphics[width=0.17\textwidth,valign=m]{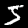}    
                        &   \includegraphics[width=0.17\textwidth,valign=m]{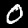}
                        &   \includegraphics[width=0.17\textwidth,valign=m]{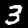}
                        &   \includegraphics[width=0.17\textwidth,valign=m]{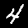}
                        &   \includegraphics[width=0.17\textwidth,valign=m]{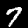}\\[-3.2em]
\rothead{\centering{Adv}} 
                        &   \includegraphics[width=0.17\textwidth,valign=m]{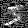}
                        &   \includegraphics[width=0.17\textwidth,valign=m]{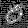}
                        &   \includegraphics[width=0.17\textwidth,valign=m]{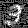}
                        &   \includegraphics[width=0.17\textwidth,valign=m]{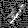}
                        &   \includegraphics[width=0.17\textwidth,valign=m]{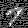}\\[-3.2em]
\rothead{\centering{Mask}} 
                        &   \includegraphics[width=0.17\textwidth,valign=m]{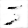}
                        &   \includegraphics[width=0.17\textwidth,valign=m]{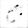}
                        &   \includegraphics[width=0.17\textwidth,valign=m]{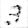}
                        &   \includegraphics[width=0.17\textwidth,valign=m]{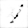}
                        &   \includegraphics[width=0.17\textwidth,valign=m]{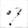}\\[-3.2em]
\rothead{\centering{Cali Adv}} 
                        &   \includegraphics[width=0.17\textwidth,valign=m]{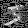}
                        &   \includegraphics[width=0.17\textwidth,valign=m]{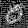}
                        &   \includegraphics[width=0.17\textwidth,valign=m]{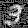}
                        &   \includegraphics[width=0.17\textwidth,valign=m]{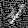}
                        &   \includegraphics[width=0.17\textwidth,valign=m]{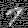}\\[-2.5em]
\end{tabular}
\captionof{figure}{Visualization of soft mask $M$. }
\label{Fig:visMask}
\end{flushright}
\end{minipage}
\end{figure}

\subsubsection{Training with Larger Perturbation Bound}\label{large perturbations}
Our method adapts adversarial examples for mitigating the adverse effect, which enables a model trained with larger perturbations. To verify the performance, we conduct experiments on PreAct ResNet-18 models trained with $\epsilon=8,9,10,11,12$ respectively and test them on CIFAR-10 test set. Baselines are trained with their official codes. Results are showed in Figure~\ref{fig:epsilons}. From Figure~\ref{fig:eps:a}, it can be observed that our method has a clearly increasing trend on robust accuracy with the increase of $\epsilon$. From Figure~\ref{fig:eps:b}, we can see that the sum of robust accuracy and natural accuracy has a slightly decreasing trend for our method, indicating a trade-off between robust accuracy and natural accuracy. However, our method's descending grade is lower than Trades and AT, which also verifies that our method has a good trade-off between robust accuracy and natural accuracy.
\begin{figure}[htb]
\centering
\vspace*{-0.1in}
\subfigure[Robust]{\includegraphics[width=0.4\textwidth]{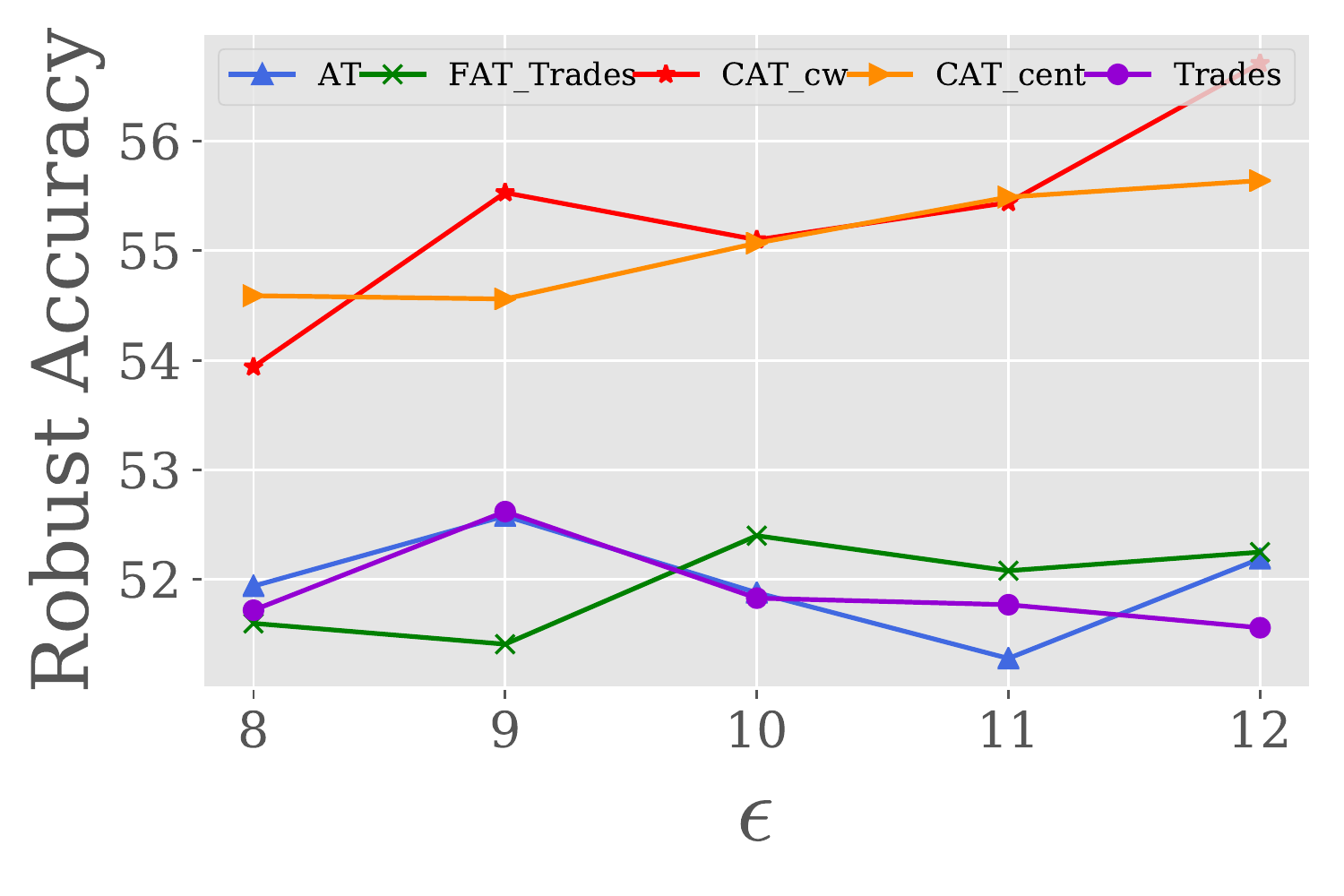} \label{fig:eps:a}
}
\subfigure[Robust+Natural]{\includegraphics[width=0.4\textwidth]{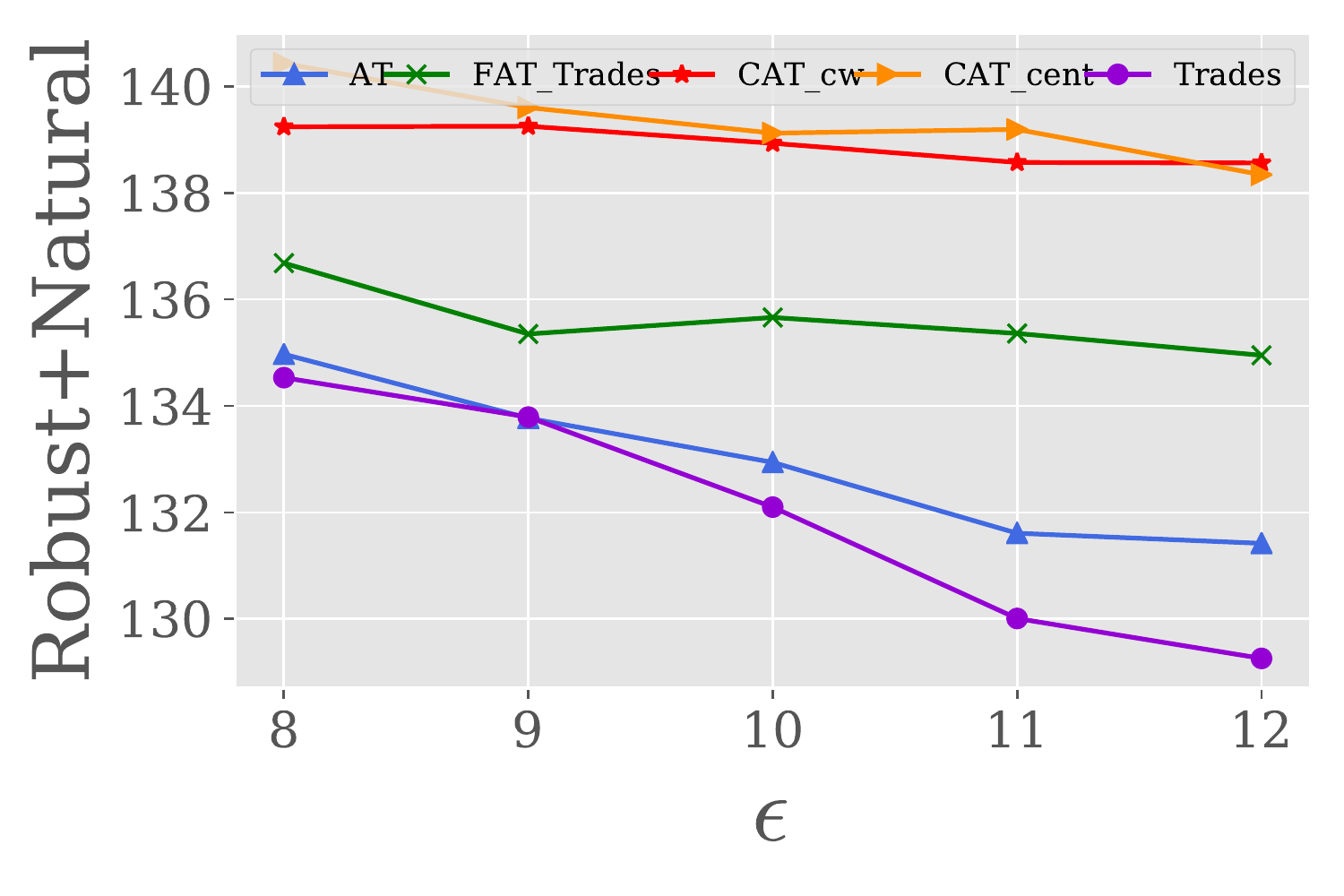} \label{fig:eps:b}
}
\caption{Evaluation on models trained with larger $\epsilon$. Robust accuracy is calculated by PGD-100 attack without random start.$\beta_{1}$ is fixed to 0.1 for $CAT_{cw}$ and $CAT_{cent}$.}
\vspace*{-0.2in}
\label{fig:epsilons}
\end{figure}

\subsubsection{Ablation Study}
We empirically verify the effect of weighted cross-entropy loss and soft mask $M$. Besides, we compare the effect of different loss functions selected in Eq.~\ref{eq.6} for generating adversarial examples. 

\textbf{Effect of the weighted cross-entropy loss and mask $M$} We remove $M$ by replacing  $X_{cali}'$ with $X'$ and remove $L(f_{\theta}(X),Y)\cdot(1-P(Y=y|X))$ by replacing it with $L(f_{\theta}(X),Y)$. We train PreAct ResNet-18 models based on CAT$_{cent}$ with removing both weighted cross-entropy loss and $M$ (marked as A1 model), and with removing $M$ only (marked as A2 model). We plot natural accuracy and robust accuracy on CIFAR-10 test set. Robust accuracy is computed by PGD-10 with random start ($\alpha=2/255,\epsilon=8/255$). Results are reported in Figure~\ref{fig:ablation_componments}. It can be observed that after removing soft mask $M$, there is a clearly decrease in natural accuracy and overall performance (natural+robust accuracy). Furthermore, after removing weighted cross-entropy loss, there is a slight decrease in natural accuracy.  

\textbf{Comparison of different loss functions} There are many choices for the surrogate loss in Eq.~\ref{eq.6} used to generate adversarial examples, e.g., cross-entropy loss, KL divergence used in Trades~\cite{pmlr-v97-zhang19p}, CW$_\infty$ loss. Here we evaluate the effect of these three losses in our method. We plot robust accuracy on CIFAR-10 test set for $\beta_{1}=0.1$ and $\beta_{1}=0.05$ respectively, and robust accuracy is calculated by PGD-10 attack with random start ($\alpha=2/255,\epsilon=8/255$). The experiments are based on PreAct ResNet-18 model. Results are showed in Figure~\ref{fig:ablation_losses} and it can be seen that $\mathbf{KL}$ divergence is less effective in achieving robustness than cross-entropy loss and CW$_\infty$ loss for both $\beta_{1}=0.1$ and $\beta_{1}=0.05$ settings. 

\begin{figure}[htb]
\centering
\vspace{-0.1in}
\subfigure[Natural]{\includegraphics[width=0.4\textwidth]{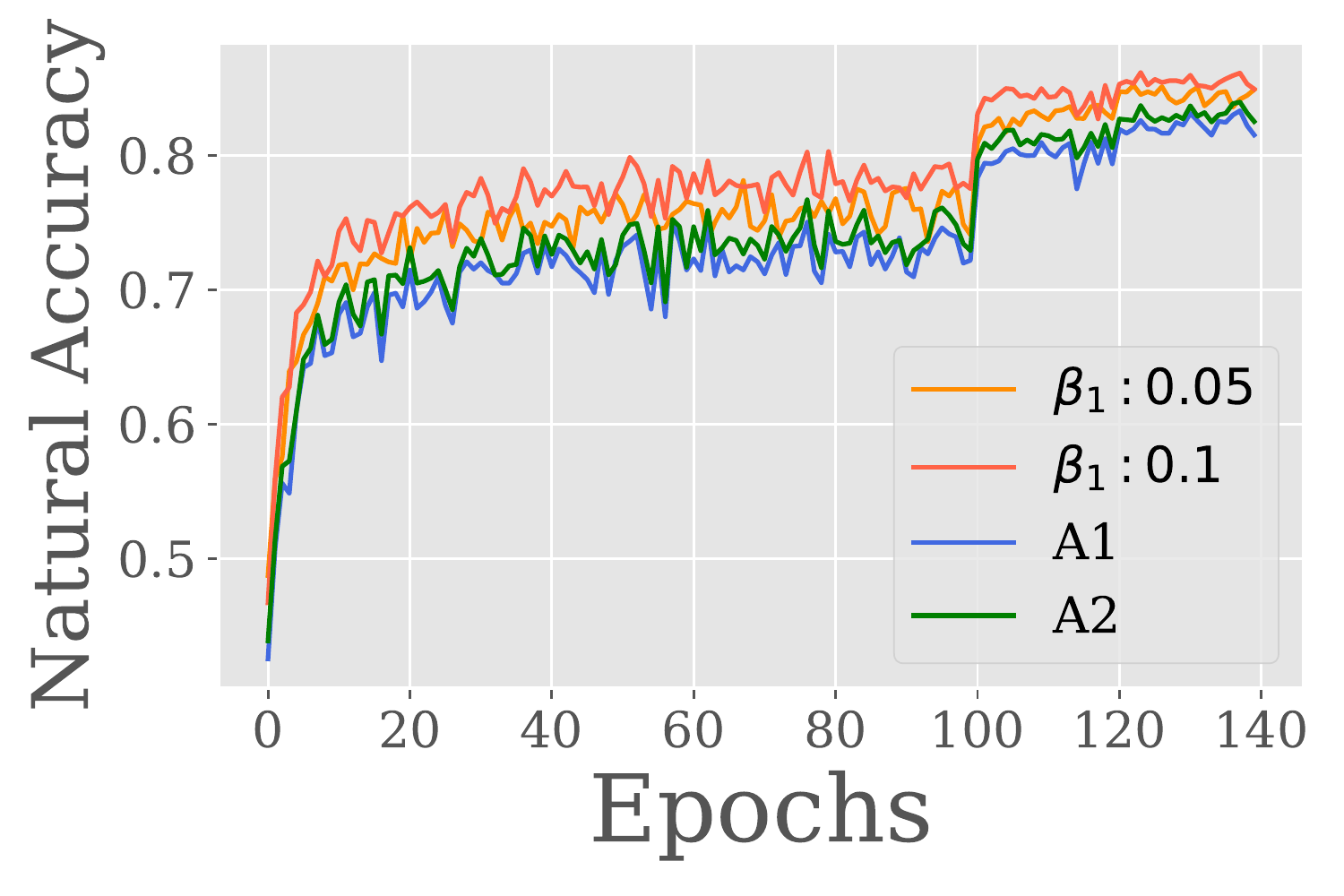}
}
\subfigure[Natural+Robust]{\includegraphics[width=0.4\textwidth]{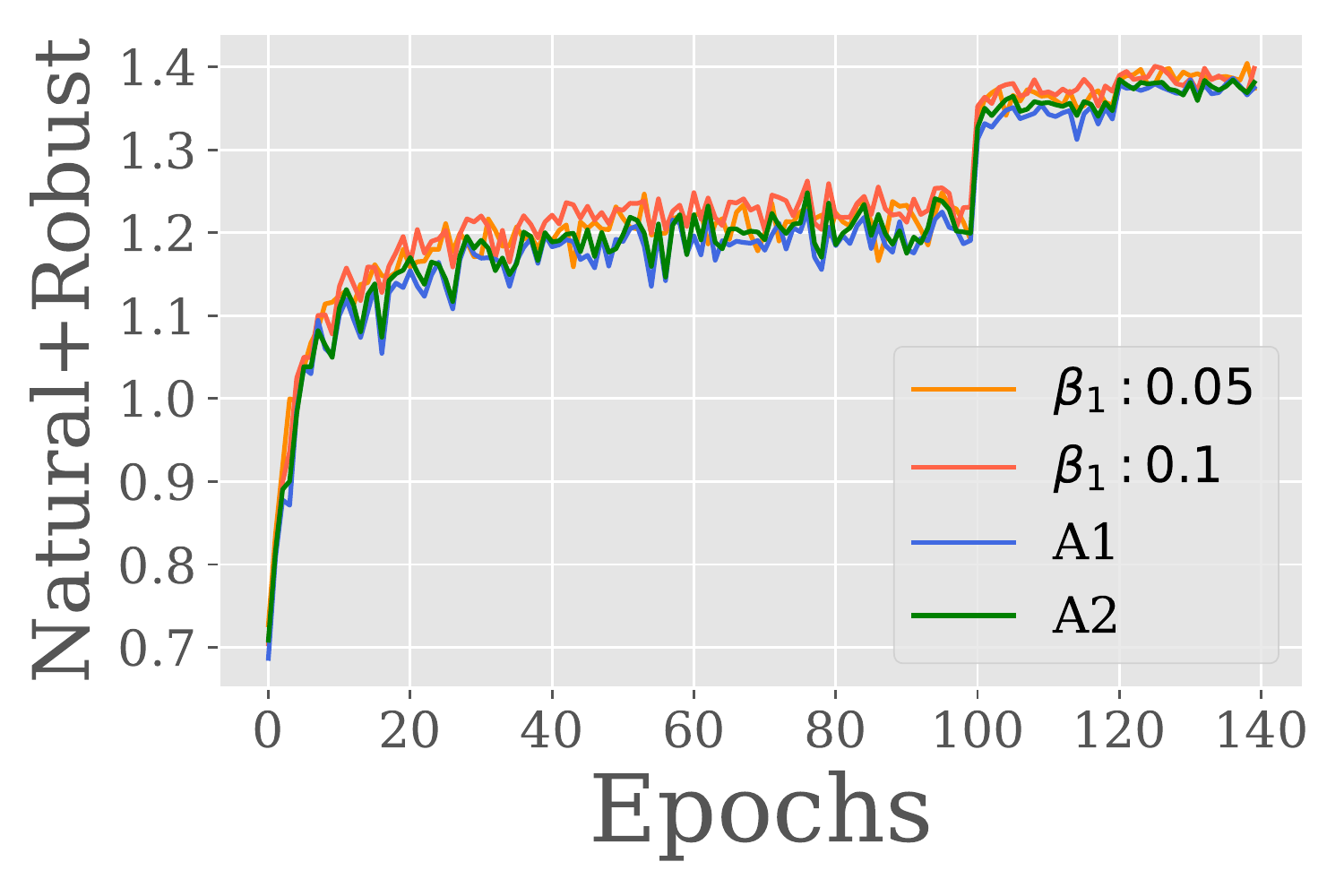}}
\caption{The ablation Experiments. A1: Model trained by CAT$_{cent}$ with removing both soft mask $M$  and $(1-P(Y=y|X))$. A2: Model trained by CAT$_{cent}$ with removing  soft mask $M$ only. $\beta_{1}:0.1,0.05$ denote models trained by CAT$_{cent}$ with setting $\beta_{1}=0.1,0.05$ respectively.}
\label{fig:ablation_componments}
\end{figure}

\begin{figure}[htb]
\centering
\vspace{-0.1in}
\subfigure[$\beta_{1}=0.05$]{\includegraphics[width=0.4\textwidth]{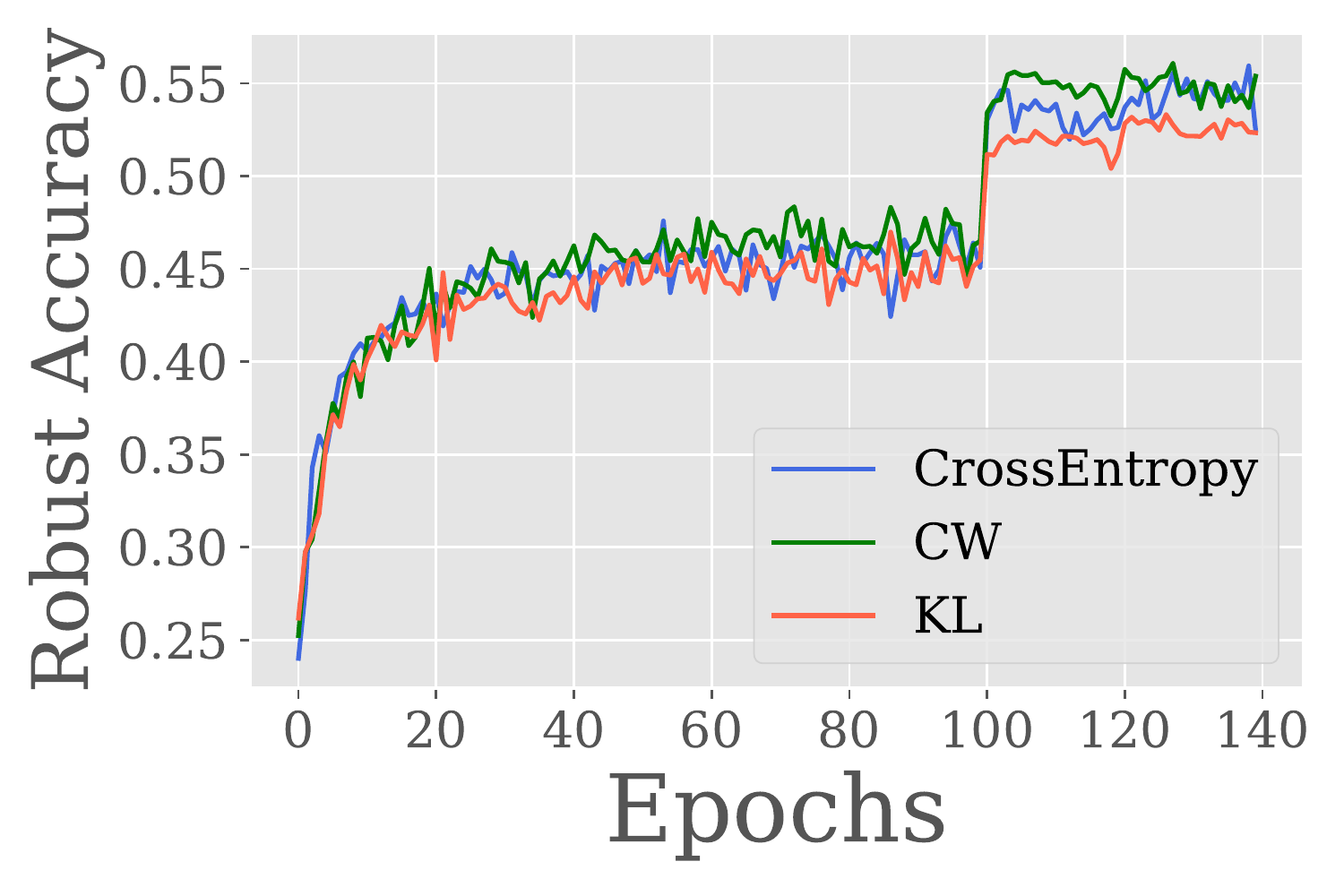}
}
\subfigure[$\beta_{1}=0.1$]{\includegraphics[width=0.4\textwidth]{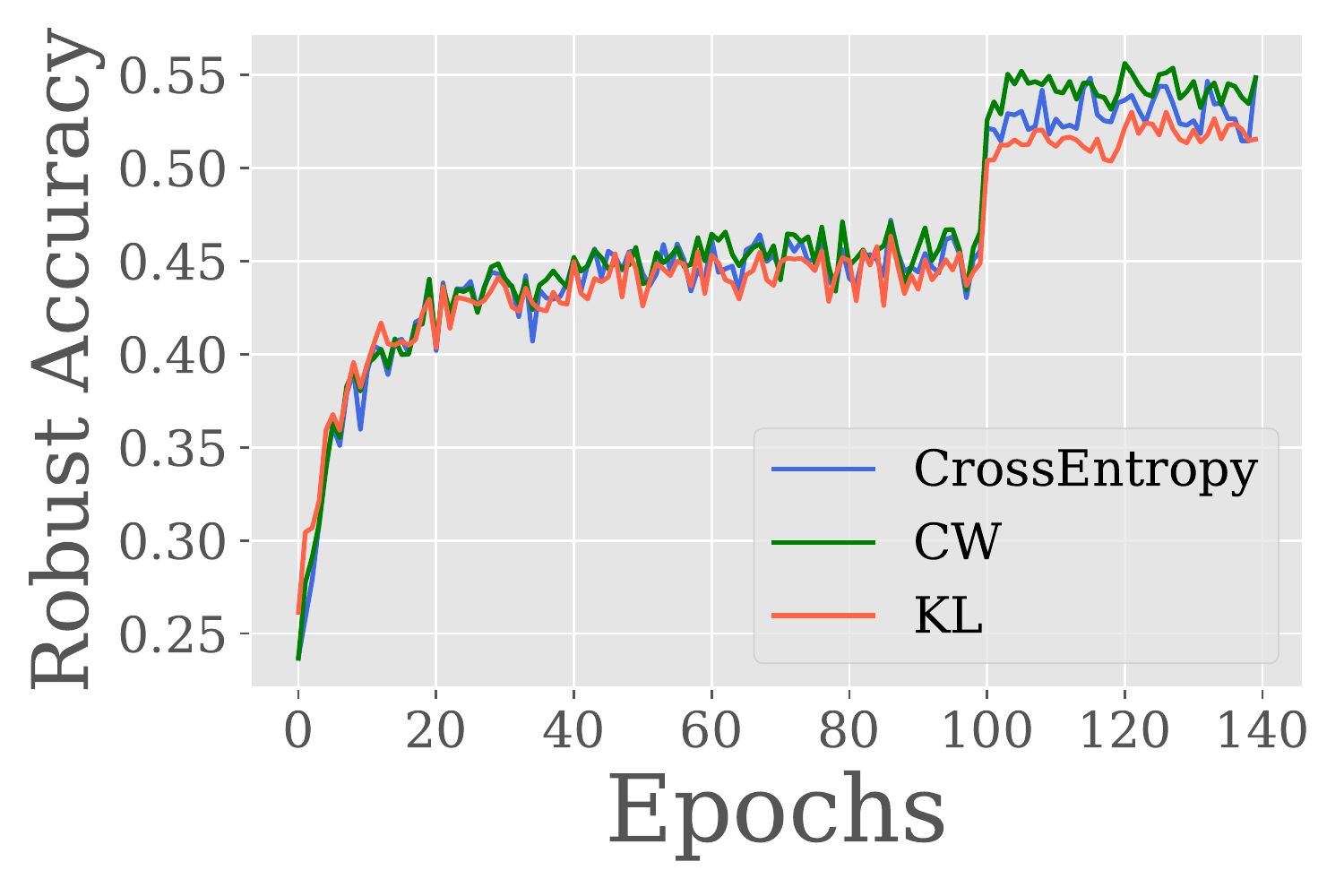}}
\caption{Comparison of different loss functions on achieving adversarial robustness.}
\label{fig:ablation_losses}
\end{figure}

\subsubsection{Analysis for Hyper-parameter $\beta_{1}$}
There are two hyper-parameters, $\beta$ and $\beta_{1}$, in our method. $\beta$ has the same effect as $\lambda$ in MART~\cite{wang2019improving} and Trades~\cite{pmlr-v97-zhang19p}. It controls the strength of the regularization for robustness. The analysis for $\beta$ can be found in Appendix~\ref{aly_beta}. $\beta_{1}$ controls the strength that pushes calibrated adversarial examples to be the same class of the input X. In this section, we mainly show the effect of $\beta_{1}$ on robust accuracy and natural accuracy. We train models with $\beta_{1}$ varying from 0.001 to 0.3 based on PreAct ResNet-18 architecture. The robust accuracy is calculated on CIFAR-10 test set by PGD-20 attack without random start. 

The trends are showed in Figure~\ref{fig:hyperAnalaysis}. 
The concrete values can be found in Table~\ref{tab:beta} (Appendix~\ref{aly_beta}). From Figure~\ref{fig:hyperAnalaysis}, it can be observed that when increasing the value of $\beta_{1}$, natural accuracy has remarkable growth. Meanwhile, PGD+Natural accuracy increases when $\beta_{1}$ is from  $0.01$ to $0.1$, which implies that calibrated adversarial examples release the negative effect of adversarial examples to some degree. With continuously increase $\beta_{1}$, there is a large drop in robust accuracy. It is because a large $\beta_{1}$ will reduce adversarial perturbation strength. However, it can be observed that there is a good trade-off for large $\beta_{1}$ between natural accuracy and robust accuracy. For example, with $\beta_{1}=0.3$, CAT$_{cw}$ achieves $88.08 \pm 0.07$ for natural accuracy while keeps $51.46 \pm 0.11$ for robust accuracy, which is much better than the trade-off achieved by Trades~\cite{pmlr-v97-zhang19p} where natural accuracy is 87.91 and robust accuracy is 41.50~\footnote{Results are copied from~\cite{pmlr-v97-zhang19p}}. 

\begin{figure}[htb]
\centering
\vspace{-0.1in}
\subfigure[CAT$_{cent}$]{\includegraphics[width=0.4\textwidth]{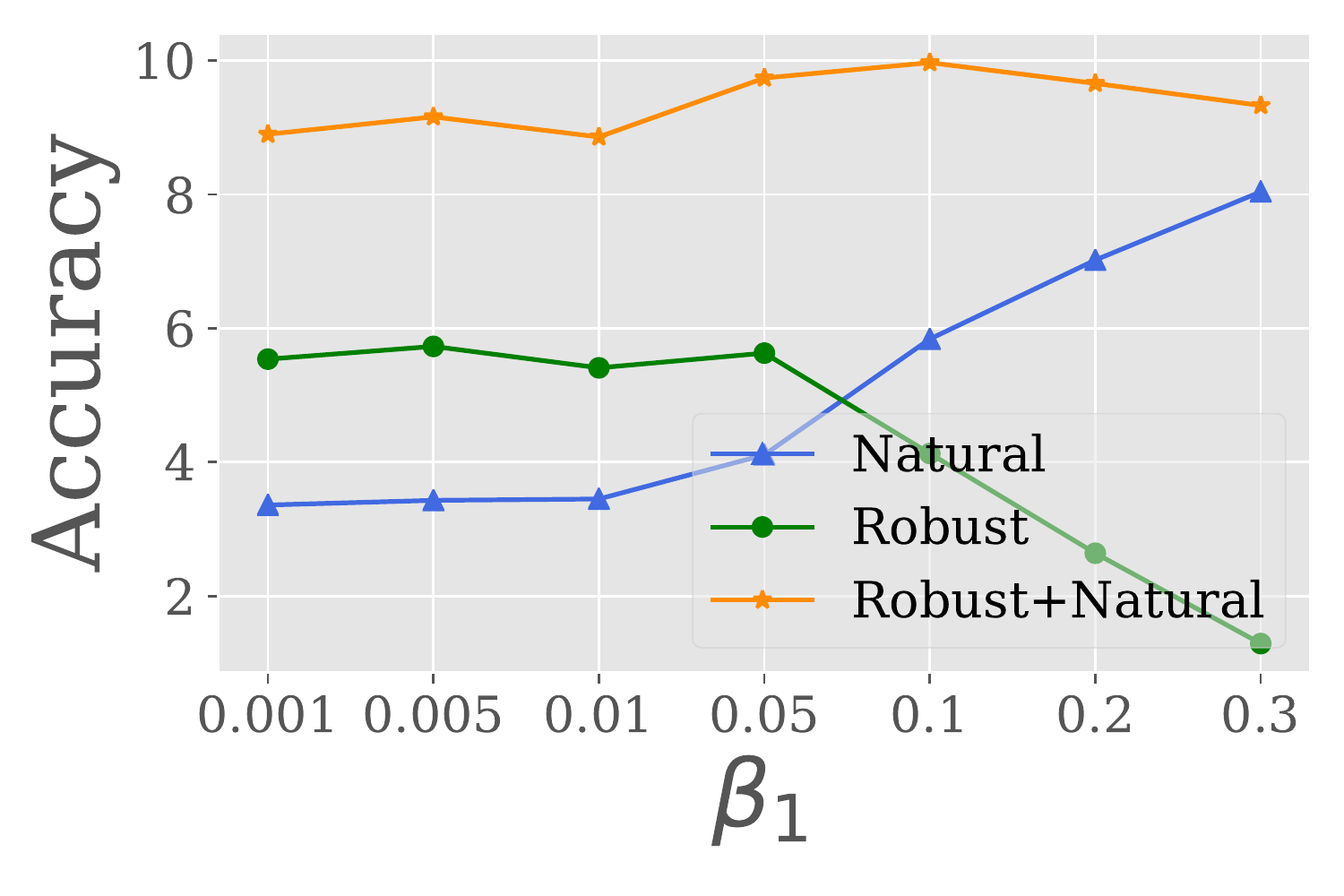}
}
\subfigure[CAT$_{cw}$]{\includegraphics[width=0.4\textwidth]{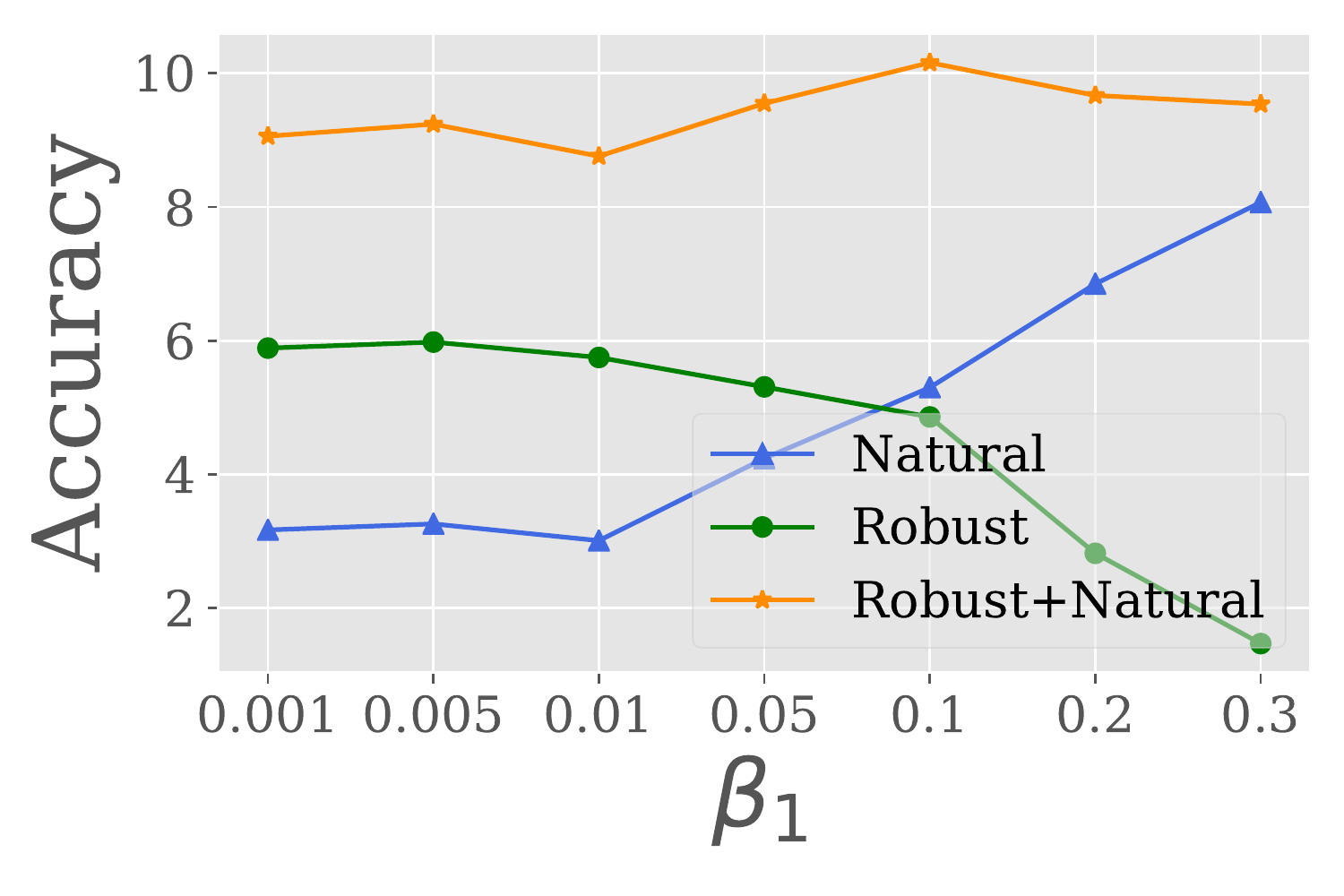}
}
\caption{Impact of hyper-parameter $\beta_{1}$ on the performance of natural accuracy and robust accuracy. Note: The natural accuracy showed in the figure is $(natural\; accuracy -80)$ and the robust accuracy showed in the figure is $(robust\; accuracy-50)$.}
\label{fig:hyperAnalaysis}
\end{figure}

\section{Conclusion}
In this paper we proposed a  new definition of robust error, i.e.\ calibrated robust error for adversarial training. We derived an upper bound for it, and enabled a more effective way of adversarial training that we call calibrated adversarial training. Our extensive experiments demonstrate that the new method improves natural accuracy with a large margin, and achieves the best performance under both white-box and black-box attacks among all considered state-of-the-art approaches. Our method also has a good trade-off between natural accuracy and robust accuracy.


\bibliography{acml21}

\appendix
\section{Proof} \label{proof}
\subsection{Proof for Theorem 1}\label{proofR}

\begin{proof}
We denote the set $S_{R}=\{(X,Y)|\forall (X,Y)\sim D, \exists X' \in \mathcal{B}(X,\epsilon)\;s.t. \; f_{\theta}(X')Y \leq 0 \}$ and $S_{CaliR}=\{(X,Y) |\forall (X,Y)\sim D, \exists X' \in \mathcal{B}(X,\epsilon)\; s.t. \; f_{\theta}(X')f_{oracle}(X') \leq 0 \}$.

\textbf{ 
 Since $S_{R} \subseteq S_{CaliR}$ $\implies$ $\mathcal{R}_{rob}(f) \leq \mathcal{R}_{cali}(f)$, we only need to prove $S_{R} \subseteq S_{CaliR}$}. 
\begin{align}
&\mathbf{\forall (X,Y) \in S_{R}}, \notag \\
&\mathbf{(1) \;if \;  f_{\theta}(X)Y \leq 0},
 then \; f_{\theta}(X)f_{oracle}(X) \leq 0  \implies X \in S_{CaliR}. \notag \\
&\mathbf{(2) \;if \; f_{\theta}(X)Y>0}, then \; \exists X' \in \mathcal{B}(X,\epsilon)\; s.t\; f_{\theta}(X')Y \leq 0; \notag\\
&\; \; \; \; \mathbf{1)if \; f_{oracle}(X') f_{\theta}(X') \leq 0} \implies X \in S_{CaliR} \notag \\
&\; \; \; \; \mathbf{2) if \; f_{oracle}(X') f_{\theta}(X')>0},\; then\; it \; must \; have: \; \notag \\ 
&\; \; \; \;\; \; \; \;\;\;\;\; \exists  X^{''} \in \mathcal{B}(X,\epsilon)\; s.t.  f_{\theta}(X^{''})f_{oracle}(X^{''}) \leq 0 \label{oriprob}; 
\end{align}
\emph{\;\;\;\;\;\;\;\;\;\;\;\;\;\;\;\;\;\;\;\;\;\;\;We prove Eq.~\ref{oriprob}  by the contradiction method. We assume}:\\
\begin{align}
\begin{split}
\forall  X^{''} \in \mathcal{B}(X,\epsilon)\; s.t.  f_{\theta}(X^{''})f_{oracle}(X^{''})>0 \; is \; True. \label{localassm1}
\end{split}
\end{align}
\begin{align}
\begin{split}
\;\;\;\;\;\;\;\mathbf{f_{\theta}(X)Y > 0,f_{\theta}(X')Y \leq 0} \implies the\;decision\; boundary\;of\; f_{\theta} \notag \\  
crosses\; the\; \epsilon -norm\; ball\; of\; X.  \notag\\
\end{split}\\
\begin{split}
\;\;\;\;\;\;\;\mathbf{f_{\theta}(X')Y \leq 0, f_{oracle}(X')f_{\theta}(X')>0} \implies  the\;decision\; boundary \notag \\ of\; f_{oracle}\;crosses\;the\; \epsilon -norm\; ball\; of\; X.  \notag
\end{split}
\end{align}
\begin{align}
\begin{split}
\;\;\;\;\;\;\;&\mathbf{If\; Eq.\ref{localassm1} \; is \;true},which\; implies\;that\;  f_{\theta}\; and\; f_{oracle}\; have\;the\;\notag \\ \;\;\;\;\;\;\;&\;same\; prediction\; on\;any\;sample\;from\;the\;\epsilon\;-ball\;of\; X.\notag \\
\;\;\;\;\;\;\;&\implies \;the\; decision\;   boundaries\; of\; f_{\theta}\; and\;f_{oracle}\; will\; be\; \notag \\ 
\;\;\;\;\;\;\;&completely\; overlapped\; in \; \epsilon\;-ball \; of\;X,\;\mathbf{ which\; contradicts}\; 
 \notag \\
\;\;\;\;\;\;\;&the\; assumption\; of\;  the\; Theorem\; 1: the\;
decision\; boundaries \notag \\ 
\;\;\;\;\;\;\;& of\;  f_{\theta}\; and\; f_{oracle}\; are\; not\; overlapped.
\end{split}\\
\begin{split}
\;\;\;\;\;\;\;&\mathbf{Therefore\; Eq.\ref{localassm1}\; is\; False}\;. \notag \\
\;\;\;\;\;\;\;&\implies \exists  X^{''} \in \mathcal{B}(X,\epsilon)\; s.t.  f_{\theta}(X^{''})f_{oracle}(X^{''}) \leq 0; Eq.~\ref{oriprob}\; is\; proved.\\
\;\;\;\;\;\;\;&\implies X \in S_{CaliR}. \notag
\end{split}
\end{align}
By now, we proved $\forall X \in S_{R} \implies X \in S_{CaliR}$. Besides, $\exists X \in S_{CaliR} \implies X \notin S_{R}$, e.g. the sample $X$ in Fig.~\ref{fig:epsilonsProof:a}. Therefore $S_{R} \subseteq S_{CaliR}$ is proved.
\end{proof}
Besides going through a formal proof itself, we think it is useful to look into the provided visualization of the decision boundary for a more intuitive understanding. According the spatial relationship of decision boundaries of $f_{\theta}$ and $f_{oracle}$, it can be separated into intersection and non-intersection cases (no overlap case according to the assumption in Theorem 1), which are showed in Fig.~\ref{fig:epsilons}.  From Fig.~\ref{fig:epsilons}, for any sample (X,Y) from class 2, if $\exists X' \in \mathcal{B}(X,\epsilon)$ lies in the region filled with blue lines, it must have $\exists X^{''} \in \mathcal{B}(X,\epsilon)$ lies in the region filled with gray lines. However, if $\exists X^{'} \in \mathcal{B}(X,\epsilon)$ lies in the region filled with gray lines, it is possible that $\forall X^{'} \in \mathcal{B}(X,\epsilon)$ do not lie in the region filled with blue lines. Therefore $S_{R} \subseteq S_{CaliR}$.

\begin{figure}[htb]
\tiny
\centering
\subfigure[Intersect1]{\includegraphics[width=0.22\textwidth]{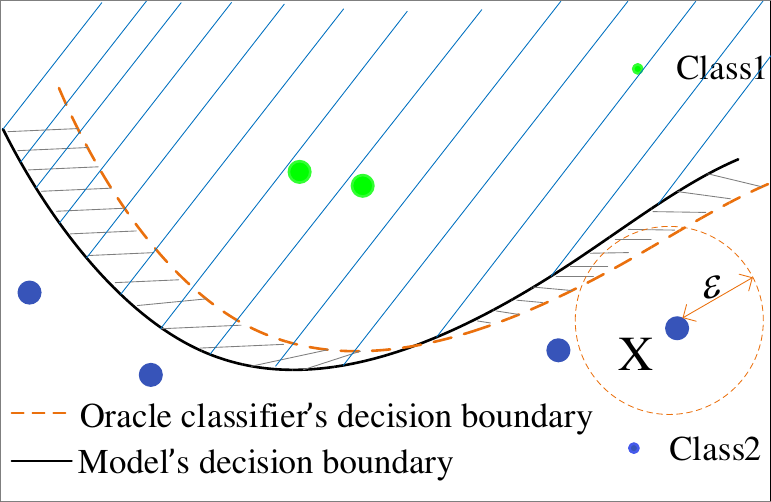} \label{fig:epsilonsProof:a}
}
\subfigure[Intersect2]{\includegraphics[width=0.22\textwidth]{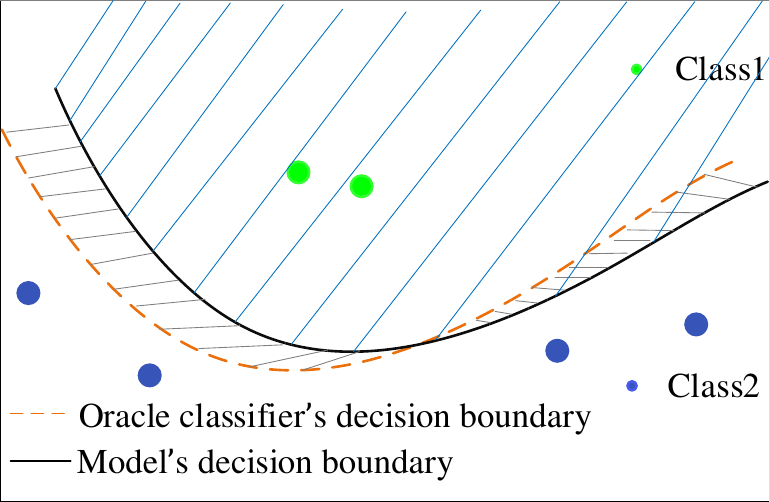} \label{fig:epsilonsProof:b}
}
\subfigure[Non-Intersect1]{\includegraphics[width=0.22\textwidth]{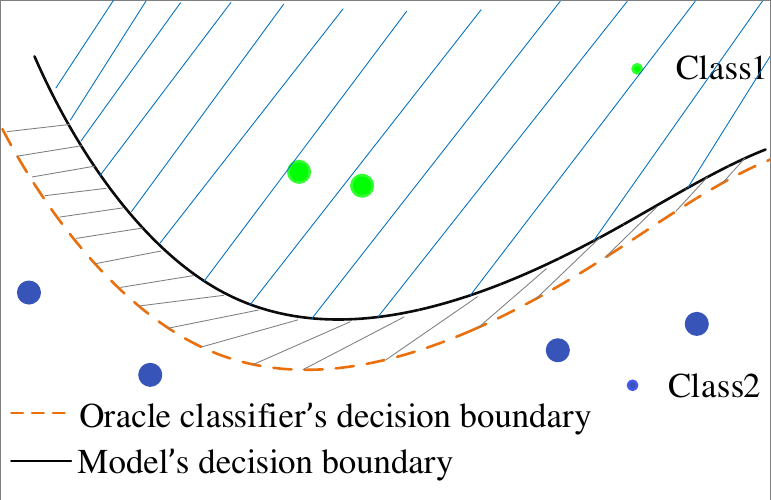} \label{fig:epsilonsProof:c}
}
\subfigure[Non-Intersect2]{\includegraphics[width=0.22\textwidth]{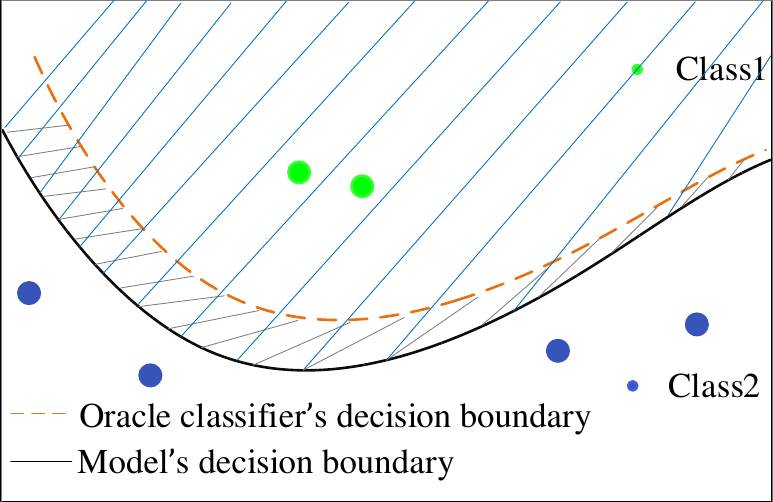} \label{fig:epsilonsProof:d}
}
\vskip -0.1in
\caption{Visualization of $f_{\theta}$ and $f_{oracle}$ decision boundaries. Region filled with gray lines: $\{ X'|f_{\theta}(X')f_{oracle}(X')\leq 0 \}$. Region filled with blue lines: $\{X'| f_{\theta}(X')Y\leq 0, Y=class2\}$.}
\label{fig:epsilonsProof}
\vskip -0.1in
\end{figure}

\subsection{Proof for Theorem 2}\label{proofTheo}
\upperbound*
\begin{proof}
\begin{align}
&\mathcal{R}_{cali}(f)-\mathcal{R}^* =\mathbf{E}_{(X,Y) \sim D} \bm{1} \{\exists X' \in \mathcal{B}(X,\epsilon) \; s.t. \; f_{\theta}(X')f_{oracle}(X') \leq 0 \}  \notag \\ 
 &=\mathbf{E}_{(X,Y) \sim D} \bm{1}\{\exists X' \in \mathcal{B}(X,\epsilon) \; s.t. \; f_{\theta}(X')f_{oracle}(X') \leq 0,f_{\theta}(X)Y \leq 0 \} \notag\\
 &\qquad \qquad +\mathbf{E}_{(X,Y) \sim D} \bm{1} \{\exists X' \in \mathcal{B}(X,\epsilon) \; s.t. \; f_{\theta}(X')f_{oracle}(X') \leq 0,f_{\theta}(X)Y > 0 \}-\mathcal{R}^*  \notag\\
 &=\mathbf{E}_{(X,Y) \sim D} \bm{1}\{f_{\theta}(X)Y \leq 0 \}-\mathcal{R}^* \notag \\
 &\qquad \qquad+\mathbf{E}_{(X,Y) \sim D} \bm{1} \{\exists X' \in \mathcal{B}(X,\epsilon) \; s.t. \; f_{\theta}(X')f_{oracle}(X') \leq 0,f_{\theta}(X)Y > 0 \} \notag \\
 &\Scale[0.9]{\leq \psi^{-1}(\mathcal{R}_{\phi}(f)-\mathcal{R}_{\phi}^{*})+\mathbf{E}_{(X,Y) \sim D} \bm{1} \{\exists X' \in \mathcal{B}(X,\epsilon) \; s.t. \; f_{\theta}(X')f_{oracle}(X') \leq 0,f_{\theta}(X)Y > 0 \}} \notag \\
 & \leq \psi^{-1}(\mathcal{R}_{\phi}(f)-\mathcal{R}_{\phi}^{*})+\mathbf{E}_{(X,Y) \sim D} \bm{1} \{\exists X' \in \mathcal{B}(X,\epsilon) \; s.t. \; f_{\theta}(X')f_{oracle}(X') \leq 0 \} \notag \\
 & \leq \psi^{-1}(\mathcal{R}_{\phi}(f)-\mathcal{R}_{\phi}^{*}) +\mathbf{E}_{(X,Y) \sim D} \max_{X' \in \mathcal{B}(X,\epsilon)} \bm{1}\{f_{\theta}(X')f_{oracle}(X') \leq 0\} \notag \\
 & \leq \psi^{-1}(\mathcal{R}_{\phi}(f)-\mathcal{R}_{\phi}^{*}) +\mathbf{E}_{(X,Y) \sim D} \max_{X' \in \mathcal{B}(X,\epsilon)} \phi(f_{\theta}(X')f_{oracle}(X')) \notag \\
 & Let \;\; f_{oracle}(X')=Y,\; then, \notag \\
 &\mathcal{R}_{cali}(f)- \mathcal{R}^*\leq \psi^{-1} (\mathcal{R}_{\phi}(f)-\mathcal{R}_{\phi}^{*})+\mathbf{E}_{(X,Y) \sim D} \max_{\substack{X'\in \mathbf{B}(X,\epsilon)\\ f_{oracle}(X')=Y}} \phi(f_{\theta}(X')Y) \notag
\end{align}
\end{proof}
The first inequality holds when $\phi$ is a classification-calibrated loss~\cite{pmlr-v97-zhang19p,bartlett2006convexity}. Classification-calibrated loss contains the cross-entropy loss, hinge loss, $\mathbf{KL}$ divergence and etc.

\section{Training Strategy}\label{algo}
There are two neural networks to be trained: $f_{\theta}$ and $g_{\varphi}$. $f_{\theta}$ is the neural network that we want to obtain and $g_{\varphi}$ is the auxiliary neural network for generating soft mask $M$. In practice, we train these two neural networks in turn. Specifically, in each step, we firstly update $f_{\theta}$ using Eq.~\ref{f_objective}, then update $g_{\varphi}$ using Eq.~\ref{g_objective}. More details can be found in Algorithm~\ref{alg:training}.

The pseudocode of our method is showed in Algorithm~\ref{alg:training}.
\begin{algorithm}[tb]
   \caption{Calibrated adversarial training}
   \label{alg:training}
\begin{algorithmic}
   \STATE {\bfseries Input:} neural network $f_{\theta}$, neural network $g_{\varphi}$, training dataset $(X,Y) \in D$.
   \STATE {\bfseries Output:} adversarial robust network $f_{\theta}$
   \FOR{$epoch=1$ {\bfseries to} $T$}
       \FOR{$mini$-$batch$=1 {\bfseries to} $M$}
       \STATE Generate $X'=X+\delta$ using PGD or C\&W$_{\infty}$ attack
       \STATE Obtain  $X_{cali}'$ using Eq.~\ref{eq.7}
       \STATE Update $\theta$ by back-propagating  Eq.~\ref{f_objective}
       \STATE Update $\varphi$ by back-propagating  Eq.~\ref{g_objective}
       \ENDFOR
   \ENDFOR
\end{algorithmic}
\end{algorithm}

\section{Implement Details} \label{imple}
\subsection{MNIST}
We  copy the model architecture for $f_{\theta}$ from \url{https://adversarial-ml-tutorial.org/adversarial_training/}.

\textbf{Model architecture for $f_{\theta}$ and $g_{\varphi}$}: 
The architecture of $f_{\theta}$ and $g_{\varphi}$ are showed in Table~\ref{archi:mnist}.

\begin{table*}[h!]
\caption{Model architecture (MINST)}
\vspace{-0.2 in}
\label{archi:mnist}
\begin{center}
\begin{small}
\begin{sc}
\begin{tabular}{cccc}
\toprule
\multicolumn{2}{c}{$f_{\theta}$}&\multicolumn{2}{c}{$g_{\varphi}$}\\
\toprule
 Layer name &Neurons  & Layer name &Neurons \\
\midrule
Conv layer &32 &  Conv layer& 64  \\
Conv layer&32  &Conv layer&128   \\
Conv layer&64 &Conv layer& 128 \\
Conv layer&64 &Up sampling& (28*28) \\
FC layer&(7*7*64)X100 &Conv layer  & 1  \\
FC layer& 100X10&Sigmoid & - \\
\bottomrule
\end{tabular}
\end{sc}
\end{small}
\end{center}
\vskip -0.1in
\end{table*}

\textbf{Hyper-parameters settings for training our method:} Epochs:40, optimizer:Adam, The initial learning rate is 1e-3 divided by 10 at 30-th epoch. we set $k=150$ for $CW_{\infty}$ loss in  CAT$_{cw}$. Other hyper-parameters are described in Section~\ref{Hyperps}. 

\subsection{CIFAR-10/CIFAR-100}
\textbf{Model architecture for $g_{\varphi}$}:
The architecture of $g_{\varphi}$ is showed in Table~\ref{archit:cifar10_100}.
\begin{table*}[htp]
\caption{Model architecture $g_{\varphi}$ (CIFAR-10/CIFAR-100)}
\vspace{-0.2 in}
\label{archit:cifar10_100}
\begin{center}
\begin{small}
\begin{sc}
\begin{tabular}{cc}
\toprule
 Layer name &Neurons  \\
\midrule
ResNet-18 without FC layer & -  \\
Up Sampling&(32*32)   \\
Conv layer& 3 \\
Sigmoid & -\\
\bottomrule
\end{tabular}
\end{sc}
\end{small}
\end{center}
\vskip -0.1in
\end{table*}

\textbf{Hyper-parameters for training our method}: 
For CIFAR-10, we use SGD optimizer with momentum 0.9, weight decay 5e-4 and an initial learning rate of 0.1, which divided by 10 at 100-th and 120-th epoch. Total epochs:140.

For CIFAR-100,we use SGD optimizer with momentum 0.9, weight decay 5e-4 and an initial learning rate of 0.1, which divided by 10 at 100-th and 110-th epoch. Total epochs:120.

we set $k=50$ for $CW_{\infty}$ loss in CAT$_{cw}$. Other hyper-parameters are described in Section~\ref{Hyperps}. 

\textbf{Baselines}
We run the official code for baselines and all hyper-parameters are set to the values reported in their papers.
\begin{itemize}
    \item \textbf{AT}, which is described in Section~\ref{sat}. We adopt the implementation in ~\cite{rice2020overfitting} with early stop, which can achieve better performance. we use the implementation in~\cite{rice2020overfitting}. \url{https://github.com/locuslab/robust_overfitting}.
    \item \textbf{Trades}~\cite{pmlr-v97-zhang19p}.It separates loss function into cross-entropy loss for natural accuracy and a regularization terms for robust accuracy. The official code: \url{https://github.com/yaodongyu/TRADES}.
    \item  \textbf{MART}~\cite{wang2019improving}. It incorporates an explicit regularization into the loss function for misclassified examples. The official code: \url{https://github.com/YisenWang/MART}.
    \item  \textbf{FAT}~\cite{pmlr-v119-zhang20z}. It constructs the objective function based on adversarial examples that close to the model's decision boundary. We adopt \emph{FAT for TRADES} as a baseline since it achieves better performance. The official code:\\ \url{https://github.com/zjfheart/Friendly-Adversarial-Training}.
\end{itemize}
All trained models in our experiments are trained in a single Nvidia Tesla V100 GPU and selected at the best checkpoint where the sum of robust accuracy (under PGD-10) and natural accuracy is highest.

\section{Experiments on CIFAR-100}\label{exp100}
This section shows the performance of our method on CIFAR-100. The test settings are the same as Table~\ref{tab:RN_CIFAR} and Table~\ref{tab:WRN_CIFAR}. Results are reported in Table~\ref{tab:cifar100}. From Table~\ref{tab:cifar100}, we can see that our method improves natural accuracy and robust accuracy compared with AT, which further shows the evidence that our method are effective in achieving natural and robust accuracy.

\begin{table}[ht]
\caption{Evaluation on CIFAR-100 (PreAct ResNet-18).}
\vspace{-0.2 in}
\begin{center}
\begin{small}
\begin{sc}
\begin{tabular}{lcccccc}
\toprule
Models & Natural &FGSM& PGD-20 & PGD-100 &C\&W$_{\infty}$&Avg\\
\midrule
AT &55.13 & 29.77&27.51&27.01&25.91& 33.06\\
CAT$_{cent}$($\beta_{1}=0.05$)&58.52&31.45 &28.93&28.48&25.55& 34.59\\
CAT$_{cent}$($\beta_{1}=0.1$) &59.77 &30.33 &27.16  &26.62&24.17&33.61\\
CAT$_{cw}$($\beta_{1}=0.05$)&58.9&\textbf{31.51} &\textbf{29.11} &\textbf{28.5}&\textbf{26.25}& \textbf{34.85}\\
CAT$_{cw}$($\beta_{1}=0.1$) & \textbf{60.37}&31.04 &28.31 &27.88&25.67&34.65 \\
\bottomrule
\end{tabular}
\end{sc}
\end{small}
\end{center}
\label{tab:cifar100}
\vskip -0.1in
\end{table}

\section{Analysis for Hyper-parameter $\beta$}\label{aly_beta}
In this section, we conduct experiments for hyper-parameter $\beta$ with varying from 1 to 5. Test settings are the same as Figure~\ref{fig:hyperAnalaysis}.  Results are showed in  Table~\ref{tab:beta}. Besides, we also report results for hyper-parameter $\beta_{1}$ with varying from 0.05 to 0.3. 

From Table~\ref{tab:beta}, it can be observed that $\beta$ also controls a trade-off between natural accuracy and robust accuracy. The larger $\beta$ value leads to a larger robust accuracy but with a smaller natural accuracy. By comparing with $\beta_{1}$, we can see that adapting $\beta_{1}$ can achieve a better trade-off than adapting $\beta$. Therefore, for our method, we suggest to fix $\beta$ to large value, e.g., $\beta=5$, then adapt $\beta_{1}$ to achieve the trade-off that we want.  

\begin{table}[h!]
\caption{Impact of hyper-parameter $\beta$.}
\vspace{-0.2 in}
\begin{center}
\begin{small}
\begin{sc}
\begin{tabular}{cccccc|cccccc}
\toprule
\multicolumn{6}{c|}{CAT$_{cent}$}&\multicolumn{6}{c}{CAT$_{cw}$} \\
\toprule
$\beta$$^a$ & $Nat$ &PGD-20 & $\beta_{1}$$^b$ &$Nat$ & PGD-20 & $\beta$$^a$ & $Nat$ &PGD-20 & $\beta_{1}$$^b$ &$Nat$ & PGD-20\\
\midrule
 5&84.1 &55.6 &0.05&84.1&55.6&5 &84.2&55.2 &0.05&84.2&55.2\\
 4&85.0 &54.6&0.1&85.8&54.1&4 &84.5&55.1&0.1&85.3&54.9\\
 3&85.4 &53.7&0.2&87.0&52.6&3&85.5&53.5&0.2&86.9&52.8\\
 2&86.4 &51.7&0.3&88.0&51.1&2 &86.3&52.6&0.3&88.1&51.5\\
 1&86.8&51.1&-&-&-&1&87.6&50.8&-&-&-\\
\bottomrule
\end{tabular}
\end{sc}
\end{small}
\end{center}
\footnotesize{$a:$ Model is trained with fixing $\beta_{1}:0.05$. $b:$ Model is trained with fixing $\beta:5$. $Nat$ denotes Natural accuracy. }
\label{tab:beta}
\vskip -0.1in
\end{table}

\section{Difference Between Pixel-level Adapted and Instance-level Adapted Adversarial Examples}\label{difference}
It is assumed that we want to find adversarial examples on the decision boundary. As showed in Figure~\ref{fig:difference_adapt}, given the maximum perturbation bound is $\varepsilon$. Instance-level adapted adversarial examples will reduce $\varepsilon$ to $\varepsilon'$ and find the adapted adversarial examples $x_{0}'$ while pixel-level adapted adversarial examples may find $x_{1}'$ or $x_{2}'$ as long as it is on the decision boundary within the $\epsilon$-ball. In other words, pixel-level adapted adversarial examples could lead to more diversified adversarial examples.

\begin{figure}[h!]
\centering
\vspace{-0.2in}
\includegraphics[width=0.6\textwidth]{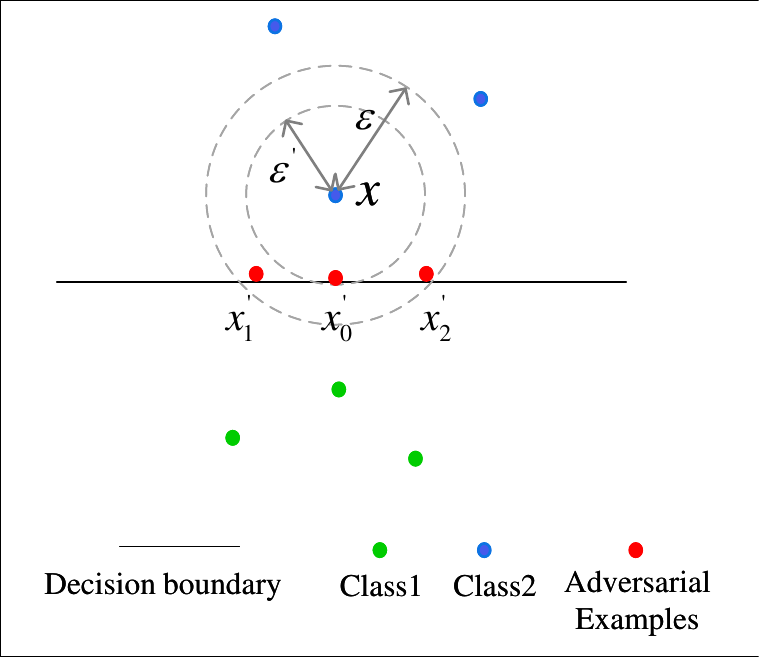}
\caption{Illustration of the difference between pixel-level adapted and instance-level adapted adversarial examples.}
\label{fig:difference_adapt}
\end{figure}

\section{Analysis for Differences Between the Proposed General Objective Function (Eq.~\ref{primalobj}) and the General Objective Function in~\cite{pmlr-v97-zhang19p}} \label{anlgeneralobj}
The general objective function derived on the upper bound of robust error is expressed as follows~\cite{pmlr-v97-zhang19p}:
\begin{align}
    \min_{\theta} \mathbf{E}_{(X,Y)} [\phi(f_{\theta}(X)Y)+\max_{X'\in \mathbf{B}(X,\epsilon)} \phi(f_{\theta}(X')Y)/\lambda ] \label{obj_robust}.
\end{align}
By comparing Eq.~\ref{obj_robust} and Eq.~\ref{primalobj}, the main difference is that there is an extra constraint $f_{oracle}(X')=Y$ for the inner maximization in our proposed general objective function. Therefore, we heuristically analyze the decision boundaries learned by these two general objective functions according to whether the constraint $f_{oracle}(X')=Y$ is active or not:
\begin{itemize}
    \item Given the oracle classifier's decision boundary does not cross the $\epsilon$-ball of input $X$, then the $f_{oracle}(X')=Y$ is an inactive constraint for the inner maximization and the proposed general objective function will be equivalent to Eq.~\ref{obj_robust}. As showed in Figure~\ref{fig:generalobj:b}, by minimizing the general objective function, the decision boundary would be transformed from black solid line to gray solid line in order to classify adversarial examples (red points) as ``Class 2''.
    \item Given the oracle decision boundary crosses the $\epsilon$-ball of input $x$, then the $f_{oracle}(X')=Y$ is an active constraint. As showed in Figure~\ref{fig:generalobj:a}, the constraint $f_{oracle}(X')=Y$ is active for input $x_{1}$. Therefore, the example generated by the inner maximization in the proposed general objective function will be $x_{1}^{*}$ (orange point) while the example generated by the inner maximization in Eq.~\ref{obj_robust} will be $x_{1}'$ (red point). By minimizing the general objective function, the decision boundary learned by the proposed general objective function could be the gray line since it try to classify $X_{1}^{*}$ as ``Class 2'' while the decision boundary learned by Eq.~\ref{obj_robust} could be the red line since it try to classify $x_{1}'$ as ``Class 2''.     
\end{itemize}
Intuitively, our proposed general objective function will push the learned decision boundary to be near the oracle classifier's decision boundary while the general objective function in~\cite{pmlr-v97-zhang19p} will push the learned decision boundary to be near the boundary of $\epsilon$-ball. 

\begin{figure}[htb]
\centering
\subfigure[Active]{\includegraphics[width=0.45\textwidth]{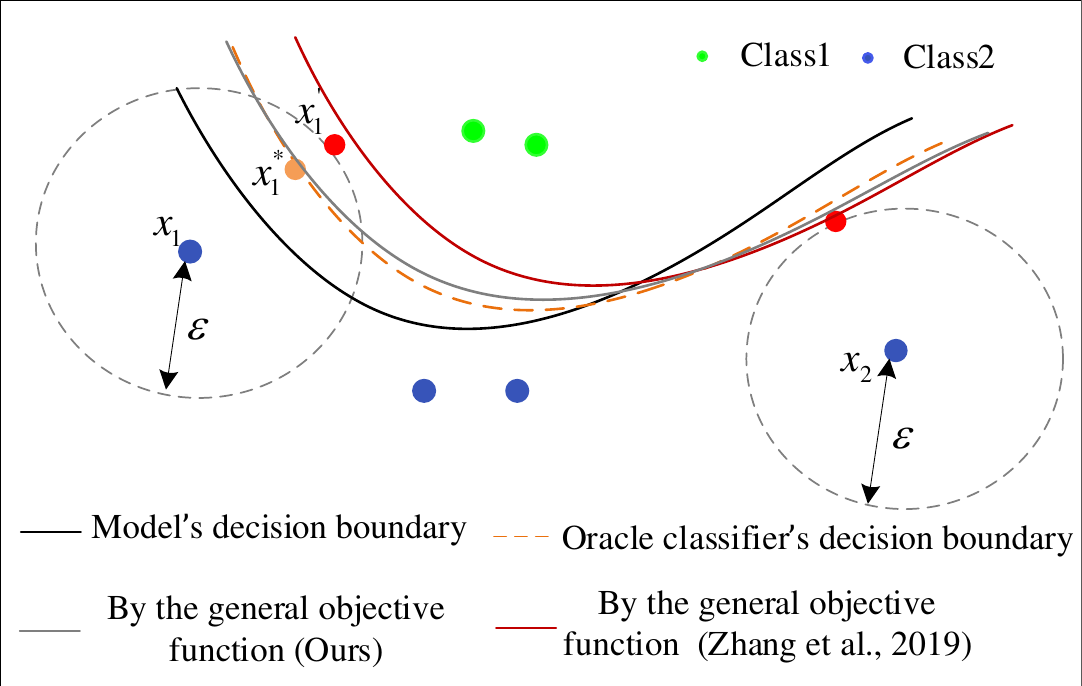}\label{fig:generalobj:a}}
\hspace{1em}
\subfigure[Inactive]{\includegraphics[width=0.45\textwidth]{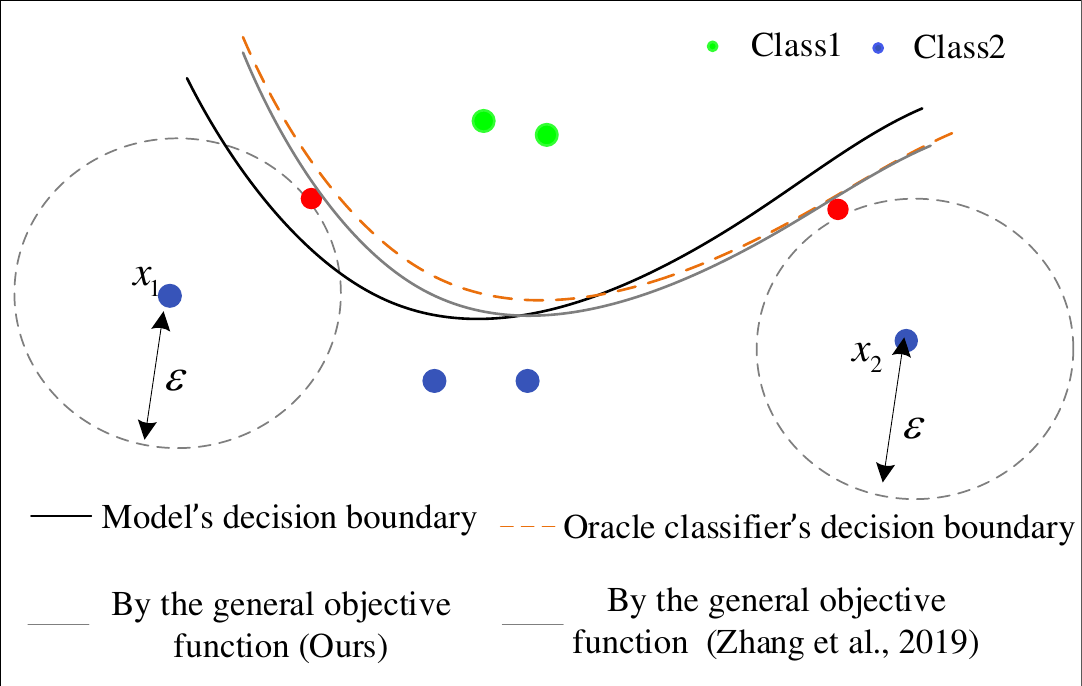}\label{fig:generalobj:b}}
\caption{Illustration for the decision boundaries learned by the proposed general objective function and the general objective function in~\cite{pmlr-v97-zhang19p}. The gray line in Figure~\ref{fig:generalobj:a} denotes the decision boundary learned by our proposed general objective function. The red line in Figure~\ref{fig:generalobj:a} denotes the decision boundary learned by the general objective function in~\cite{pmlr-v97-zhang19p}. The gray line in Figure~\ref{fig:generalobj:b} denotes the decision boundary learned by our proposed general objective function or by the general objective function in~\cite{pmlr-v97-zhang19p}.}
\label{fig:generalobj}
\end{figure}

\section{The Loss Functions for Other Variants of Adversarial Training}\label{loss_variants}
\begin{table*}[htp]
\caption{Loss functions of other variants of adversarial training.}
\vspace{-0.2 in}
\begin{center}
\begin{small}
\begin{sc}
\begin{tabular}{cc}
\toprule
 Methods & Loss Function  \\
\midrule
AT~\cite{Madry2017} & $L(f_{\theta}(X'),Y)$  \\
ALP~\cite{kannan2018adversarial} & $L(f_{\theta}(X'),Y)+\lambda \cdot \norm{f_{\theta}(X')-f_{\theta}(X)}$ \\
MMA~\cite{ding2019mma}& $L(f_{\theta}(X'),Y)\cdot \mathbf{1}(f_{\theta}(X)=Y)+L(f_{\theta}(X),Y)\cdot \mathbf{1}(f_{\theta}(X)\neq Y)$ \\
Trades~\cite{pmlr-v97-zhang19p} &$ L(f_{\theta}(X),Y)+\lambda \cdot \mathbf{KL}(P(Y|X')||P(Y|X))$\\
MART~\cite{wang2019improving} &$BCE(f_{\theta}(X'),Y)+\lambda \cdot \mathbf{KL}(P(Y|X')||P(Y|X))\cdot(1-P(Y=y|X))$\\
\bottomrule
\end{tabular}
\end{sc}
\end{small}
\end{center}
\vskip -0.1in
\end{table*}
\end{document}